\documentclass{amsart}
\usepackage{amsmath}
\usepackage{amsfonts}
\usepackage{amssymb}
\usepackage{graphicx,xcolor}
\usepackage{enumerate}
\usepackage{a4wide}
\usepackage{biblatex}
\usepackage{subcaption} 
\usepackage{placeins}
\usepackage{hyperref}   

\newtheorem{theorem}{Theorem}[section]
\newtheorem{lemma}[theorem]{Lemma}

\theoremstyle{definition}
\newtheorem{definition}[theorem]{Definition}
\newtheorem{assumption}{Assumption}
\newtheorem{remark}[theorem]{Remark}
\newtheorem{example}[theorem]{Example}

\renewcommand{\div}{\text{div}}
\newcommand{\E}{{\mathbb E}}

\newcommand{\R}{{\mathbb R}}
\newcommand{\norm}[1]{\left\|#1\right\|}

\newcommand{\N}{\mathcal{N}}
\newcommand{\Natural}{\mathbb{N}}
\newcommand{\F}{\mathcal{F}}
\newcommand{\KL}[2]{D_{KL}\left(#1 || #2\right)}
\newcommand{\argmax}{\operatorname*{argmax}}
\newcommand{\argmin}{\operatorname*{argmin}}

\newcommand{\D}{\mathcal{D}}

\title[Convex Neural Likelihood Approximation]{A Convex Approximation Framework for Neural Likelihood-Based Bayesian Inverse Problems}

\author[Schneider, Helin, Taghizadeh]{Fabian Schneider$^{1, 2}$ \and Tapio Helin$^2$ \and Leila Taghizadeh$^1$.}
\date{\today}
\thanks{L.\,Taghizadeh acknowledges support from the Austrian Science Fund (FWF), Elise-Richter grant DOI 10.55776/V1000. Funding was also received from the European Research Council (ERC) under the Horizon 2020 research and innovation program of the European Union (Grant agreement No. 101125225). This work was supported by the Research Council of Finland (353094, 348504 and 359183). Support by the Austrian Science Fund (FWF) grant I6667-N is acknowledged.
\\
$^1$ Institute of Analysis and Scientific Computing, TU Wien, Vienna, Austria \texttt{Fabian.Schneider@asc.tuwien.ac.at}, \texttt{Leila.Taghizadeh@asc.tuwien.ac.at} \\
$^2$ School of Engineering Sciences, LUT University, Lappeenranta, Finland \texttt{Tapio.Helin@lut.fi}.}

\keywords{Neural likelihood approximation, infinite dimensional Bayesian inverse problem, data-driven methods, semiconductor devices}
\subjclass[2020]{62G05, 65C05, 65N21, 65N75, 60B12}

\begin{document}

\begin{abstract}
Many problems in science and engineering are difficult to model accurately, either due to unknown physical mechanisms, poorly quantified measurement uncertainty, or prohibitive computational costs of high-fidelity simulations. These challenges limit the applicability of classical probabilistic inference methods such as Markov chain Monte Carlo, especially in high-dimensional Bayesian inverse problems. As data from scientific experiments become increasingly available, machine learning methods offer a flexible alternative to explicit parametric modelling.
We study neural likelihood approximation, where the goal is to learn the likelihood function directly from data without explicit knowledge of the underlying data-generating process. A common approach trains likelihood surrogates by minimizing the Kullback–Leibler divergence between the true posterior and an approximate posterior, which is equivalent to minimizing the expected negative log-likelihood. This work improves the theoretical foundations of neural likelihood approximation by alleviating limitations of restrictive model classes: we show that, by working with un-normalized potentials and folding normalization into the training objective, the resulting learning problem is strictly convex. We show that empirical minimizers of the resulting data-driven objective converge to the true likelihood as the sample size grows. Numerical experiments for the neural likelihood approximation are conducted for a deblurring and a non-linear PDE based imaging problem.
\end{abstract}

\maketitle


\section{Introduction}\label{s:Intro}
In many areas of science and engineering, an accurate mathematical model of the problem of interest is difficult to obtain, or there is no canonical way to quantify the uncertainty in the measurements. Even when a suitable mathematical model exists, evaluating its discretization can be computationally expensive (for example when it requires solving one or more partial differential equations (PDEs)) which may render problems of probabilistic inference such as Markov-Chain Monte Carlo (MCMC) infeasible, especially for very high dimensional problems such as Bayesian inverse problems on function spaces.
Examples of problems that fit this description include inverse problems in medical imaging \cite{arridge1999optical, Holder2005EIT} or geophysical imaging \cite{wang2025sequentialbayesiandesignefficient}, modelling in evolutionary biology \cite{Ratmann2007UsingLikelihood-free, Wood2010Statistical, Meeds2014GPS-ABC}, ground water models \cite{ZENG2025Solvinghigh-dimensional}, transmission dynamics of bacterial infections \cite{Numminen2013Estimating, Gutman2016BayesianOptimization}, financial models \cite{Ravi2004Risks, Ramona2023Bayesiansynthetic} and neurophysiological problems \cite{lueckmann2019likelihoodfree}. 
With increasing availability of measurement data from scientific experiments, machine learning methods are becoming more and more popular over hand-tuned explicit models involving a finite amount of parameters in an approximative class.

We focus on \emph{neural likelihood approximation}, where the goal is to infer the likelihood function from empirical data  without a mathematical model of the precise physics of the data generating process. 
Tuning the likelihood approximation through the KL-divergence between the posterior distribution $\mu^y$ and an approximation $\mu^y_f$ corresponding to the approximate likelihood has seen a lot of promising experiments in the literature \cite{lueckmann2019likelihoodfree, papamakarios2019sequantialneural, Balaji2017SimpleandScalable, rothfuss2019conditional}.
As is established (see e.g. \cite{papamakarios2019sequantialneural, radev2020bayesflowlearningcomplexstochastic}), this corresponds to minimising the expected negative log-likelihood (Lemma~\ref{lem: basic lemma} for details). In spite of notable achievements, the approach  continues to lack a comprehensive theoretical explanation.

In practice, normalization of the likelihood as a probability measure is often achieved through certain classes of parametric distributions \cite{lueckmann2019likelihoodfree, Balaji2017SimpleandScalable} or normalizing flows \cite{papamakarios2019sequantialneural, rothfuss2019conditional} or invertible networks \cite{radev2020bayesflowlearningcomplexstochastic}. This is restrictive: in each case the resulting function class may be too small to contain a near-optimal likelihood.

\subsection{State of the art}\label{ss:SoA}
In the context of inverse problems, the restrictive computational cost has been addressed widely through different classes of surrogate models. One of the variants corresponds to adjusting the observational noise to compensate for modelling error \cite{Arridge_2006, kaipio2006statistical, Helin2025OAE} given an approximation of the forward map. Polynomial Chaos expansion \cite{Marzouk2009Astochastic, MARZOUK2009Dimensionality, Xiu2010} is based on the idea of finding a computationally feasible approximation of the forward map. More recently, deep learning based surrogates for the forward map have been studied extensively, see e.g. \cite{Herrmann_2020,AnAdaptive2020, Li2023PINNsurrogate, Deveny2023DeepSurrogate} among many others.
Random likelihood approximation through Gaussian processes (GP) has been studied widely in the literature including statistical properties such as consistency
\cite{stuart2018PosteriorConsistency, Helin2024IntrotoGPR}. It is applicable to either an approximation of the forward map or the likelihood directly. Notice however that often a finite amount of evaluations of the likelihood is required for the estimation. 

Methods in the field of simulation based inference are applicable without a given mathematical model of the forward process or likelihood. The most classical method is given by approximate Bayesian computations (ABC) \cite{Rubin1984, Marin2012, Sisson2007Sequential}, where samples from a prior distribution are passed through a simulator until a quantity similar enough to the outcome of the physical experiment is observed.

To overcome limitations especially in terms of applicability when no mathematical forward model is available, approximation of probability distribution through its likelihood or a density from data is studied under the terms of neural likelihood approximation \cite{papamakarios2019sequantialneural, WALCHESSEN2024Neurallikelihood, lueckmann2019likelihoodfree, hikida2025multilevelneuralsimulationbasedinference} or conditional density estimation \cite{Papamakarios2016Fast,papamakarios2019neuraldensityestimationlikelihoodfree,rothfuss2019conditional}, which require only empirical data from the joint distribution of measurements and corresponding quantities of interest.
Similar variants exist for neural posterior estimation \cite{NEURIPS2022Truncatedproposals, hikida2025multilevelneuralsimulationbasedinference}.
The statistical accuracy of neural likelihood approximation under suitable assumptions are studied in \cite{frazier2024statisticalaccuracyneuralposterior}.

From a theoretical point of view, stability of the posterior with respect to the likelihood is studied extensively in the literature and Lipschitz continuity of the posterior with respect to a likelihood perturbation is established in KL-divergence, Hellinger and other distances \cite{Stuart_2010, Sprungk_2020, cvetković2025upperlowerboundslocal}. For random surrogates, the Hellinger distance between posterior and its approximation is bounded by a moment of likelihood mismatch in \cite{Lie2018RandomForward}.

\subsection{The problem setting}\label{ss:Setting}
Let the prior measure $\mu$ on a Hilbert space $H$ be given. Denote by $L: H \times\R^m \rightarrow \R$ the negative log-likelihood (NLL) or potential, describing the distribution of the dependent variable (measurement) $y \in \R^m$ conditional on the unknown independent variable (quantity of interest) $x \in H$.
By Bayes' theorem, the posterior distribution $\mu^y$ of $x$ conditional on $y$ has a density with respect to $\mu$ given by
\begin{equation}\label{eq: posterior corresponding to L}
    \frac{d \mu^y}{d \mu}(\cdot) :=\frac{1}{Z(y)} \exp \left( -L(\cdot; y) \right), 
\end{equation}
$\mu$-almost everywhere with
\begin{eqnarray*}
    Z(y) = \int_H \exp \left( -L(x; y) \right) \mu(dx).
\end{eqnarray*}
We denote by $\lambda(dx, dy)$ 
the joint distribution of $(x, y)$ and by $\pi(y)$ the marginal distribution of $y$.
The Kullback-Leibler (KL) divergence for measures $\mu \ll \nu$ on $H$ is given by
\begin{eqnarray*}
    \KL{\mu}{\nu} := \int_H \log \frac{d \mu}{ d \nu}(x) \mu(dx).
\end{eqnarray*}
Notice that the KL-divergence is not symmetric, it may take different values depending on the arrangement of its arguments.

The usual setup of (sequential) neural likelihood estimation can be summarized as follows (see \cite{papamakarios2019sequantialneural} for a detailed discussion):
In a first step, training data $(x_i, y_i)_{i = 1}^N$ from the joint distribution $\lambda$ is generated by sampling from the prior and using a simulator to compute the observation $y_i$ conditional on $x_i$ for $i = 1, \cdots, N$. We emphasize that, since in general the noise level of the simulator is unknown, the simulator cannot serve in a likelihood function.
In the next step a neural likelihood estimate $q_\theta$ is fitted by estimating
 \begin{eqnarray}\label{eq: opt for neural likelihood}
        \theta^*&:=&\argmax_{\theta \in \Theta} \E^{(x, y)} \log q_\theta(y|x) 
    \end{eqnarray}
where the generated data is used to estimate the expectation. 
By an argument similar to Lemma~\ref{lem: basic lemma} in our text,
\begin{eqnarray*}
    \argmax_{\theta \in \Theta} \E^{(x, y)} \log q_\theta(y|x) = \argmin_{\theta \in \Theta} \E^y\KL{ q(y|x) }{q_\theta(y|x)},
\end{eqnarray*}
where $q$ denotes the true likelihood.
The class $\Theta$ can be of parameters $\Theta \subset \R^d$ or of (normalized) functions such as autoregressive flows \cite{Papamakarios2079Maskedautoregressive}. 
The approximative likelihood $q_\theta$ can now be used to approximate the posterior for a given measurement $y$. In the sequential variant, data can be regenerated by sampling the $x_i$ from the approximative posterior conditioned on a given measurement $y$ instead. Multiple iterations of likelihood estimation and data generation can be performed.

This work is motivated by Bayesian inverse problems.
Important challenges include high computational cost of the evaluation of the forward map in large scale inverse problems on function spaces, where an evaluation can correspond to solving a (system of) PDE(s). A computationally cheaper approximation of the NLL $L$ is often necessary to determine unknown functions and allow uncertainty quantification by generating many posterior samples. 
Another feasible application is where no accurate mathematical formula for the forward map is available but instead limited data can be generated from a physical system.

\subsection{Our contribution}\label{ss:Contribution}
As noted above, a key difficulty in neural likelihood approximation is choosing a class of potentials that is both expressive enough and normalized as a probability distribution.
We address this as follows. Given a general function, we normalize it formally by dividing by its expectation, so that the normalization is absorbed into the expected KL-based loss. We show that this loss is convex in the un-normalized NLL, and we establish consistency of the approach: minimizers of the empirical loss converge to the true potential as the amount of data grows.
More precisely, we consider the optimization problem
    \begin{eqnarray}\label{eq: objective}
        \argmin_{f \in \F} \E^{y}\KL{\mu^y}{\mu^y_f},
    \end{eqnarray}
    where $\mu^y_f$ is the posterior corresponding to the approximate NLL:
    \begin{equation}\label{eq: posterior corresponding to f}
    \frac{d \mu^y_f}{d \mu}(x) :=\frac{\exp \left( -f(x; y) \right) }{\int_H \exp(-f(x;y)) \mu(dx)} \; \text{ in } L^1(\mu) 
\end{equation}
and $f\in \F$ in a suitable class of functions, see Definition~\ref{def: approximate posterior} for details.
Our main findings can be summarized as follows:
\begin{itemize}
    \item We remove the need to parameterize normalized likelihoods by optimizing~\eqref{eq: objective} directly over a general class of \emph{un-normalized} potentials. The key observation is the identity (Lemma~\ref{lem: basic lemma}), which reads
    \begin{eqnarray*}
     \E^\pi\KL{\mu^y}{\mu^y_f} =  C + \E^\lambda[f(x;y)] + \E^\pi[\log Z_f(y)],
    \end{eqnarray*}
with $C$ independent of $f$. Normalization is absorbed into the log-normalizer $\E^\pi[\log Z_f(y)]$, so the objective can be evaluated from samples of the joint distribution $\lambda$ without access to the true potential $L$. Minimizing the expected KL-divergence is thus equivalent to minimizing the expected negative log-likelihood together with this log-normalizer.
    \item In Theorem~\ref{thm: convex}, we show that the KL-divergence 
    \begin{eqnarray*}
        \E^{y}\KL{\mu^y}{\mu^y_f}
    \end{eqnarray*}
    is convex as a function of $f$. Exploiting this convexity of course requires the class $\F$ to be convex as well. We stress that this fails for the usual normalizing-flow parametrizations \cite{papamakarios2019sequantialneural, rothfuss2019conditional}: the set of functions $f$ for which $e^{-f}$ integrates to one is in general not convex.
    \item In Section~\ref{sec: data driven}, we derive an approximation of the objective \eqref{eq: objective} that relies solely on data from the joint distribution $\lambda$. This allows accurate and computationally efficient representations of the NLL to be learned directly from such data.
    \item In Theorem~\ref{thm: M estimation ours} we show that, under suitable assumptions, the minimizer of the data-based objective converges to the minimizer of the population objective as the number of data points increases.
    \item In Section~\ref{sec: numerics}, we present numerical experiments on two Bayesian inverse problems: a toy de-blurring problem, in which the unknown is a one-dimensional Gaussian process, and a nonlinear, PDE-based problem of estimating the doping profile of a semiconductor from voltage-current measurements.
\end{itemize}
Our work is motivated by \cite{Helin2025OAE}, where the authors show that the KL-divergence in \eqref{eq: objective} is bi-convex in the mean and covariance parameters of a Gaussian likelihood.

In summary, we propose a framework that enables offline learning purely based on empirical data, and can further be used to speed up posterior sampling once measurement data is acquired. The framework is backed by theoretical convergence guarantees in the data size. It has the advantages that the optimization problem is convex and without approximations in the objective function.

In many scientific problems, the unknown quantity is represented by a function or another infinite-dimensional parameter to a PDE. In particular in the inverse problems community, significant effort was employed to ensuring methods are independent of the choice of discretization \cite{Lehtinen1989Linear, Lassas2009InverseProblemsandImaging}. To this end, we formulate the optimization problem on (possibly infinite-dimensional) Hilbert spaces and show the fundamental properties in Lemma~\ref{lem: basic lemma} and Theorem~\ref{thm: convex} in the un-discretized setting. We state the consistency property in a finite-dimensional setting since the underlying theory on bracketing is developed mostly in the setting of function classes on euclidean spaces.

The rest of the paper is organized as follows: In Section~\ref{s:NLL approximation}, we state the main assumptions and use them to show how the objective function can be implemented. Further we show convexity of the objective function. In Section~\ref{sec: data driven}, we propose an approximate objective function based on empirical data. We show consistency of the empirical approximation. In Section~\ref{sec: numerics}, we implement our approach and illustrate our numerical results for two Bayesian inverse problems, namely a deblurring problem and a doping inverse problem in semiconductor devices. Finally, conclusions are drawn in Section~\ref{s:Conclusions}.
\section{The Negative log-likelihood approximation and main properties}\label{s:NLL approximation}
In this section, first we show that inverse problems with a Gaussian likelihood distribution satisfy certain assumptions. We define a KL-divergence based objective function for an approximative likelihood that is based on an integral in the joint distribution $\lambda$. Then, we show that this objective function is in fact convex.
Finally, we define equivalence class of (likelihood) functions with corresponding norm, where two functions are equivalent if they equal up to normalization in $y$. It is in this norm that we can expect optimal likelihood functions to be unique.

Consider a Bayesian inverse problem with a Gaussian likelihood function.
More precisely, assume to have a forward map $A: H \rightarrow \R^m$ and the observational noise covariance $\Gamma \in \R^{m \times m}$ given. The prior distribution of $x \in H$ before any measurement is made, is given by $\mu$. A measurement $y \in \R^m$ is obtained through
\begin{eqnarray}\label{eq: inverse problem}
    y = A(x) +  \varepsilon, \quad \varepsilon \sim \N(0, \Gamma).
\end{eqnarray}
The NLL corresponding to \eqref{eq: inverse problem} is given by
\begin{equation*}
        L(x, y) = \frac{1}{2} \norm{A(x) - y}_\Gamma^2.
\end{equation*}
When the true NLL is not known, we are interested in estimating one of the quantities
\begin{itemize}
    \item the NLL in a class of functions $\F$ of $(x, y)$. We denote a likelihood emulator in this class a \emph{free-form} approximation.
    \item the true forward map $A$ from a class of forward maps $\mathcal{A}$. We call an approximation in this class \emph{residual} approximation. In this class, we are approximating only the forward map and assume the observational noise level to be known.
    \item the true forward map $A$ and the observational noise $\Gamma$, from a class of forward maps $\mathcal{A}$ and observational noises $\mathcal{G}$, respectively. This approximation we call \emph{calibrated residual} approximation.
\end{itemize}
In the former case, the assumptions on the class $\F$ for this work are expressed in Assumption~\ref{ass: assumption}. In the latter cases, we will elaborate the necessary assumptions on the classes $\mathcal{A}$ and $\mathcal{G}$ in Examples~\ref{ex: gaussian likelihood 1} and~\ref{cor: assumptions} to satisfy Assumption~\ref{ass: assumption}.

\subsection{Assumptions}
Let the following assumptions, which are an extension of the assumptions in \cite{helin2025bayesianoptimalexperimentaldesign} to a general potential, hold throughout this section. 
\begin{assumption}
    \label{ass: prior}
    The prior distribution $\mu$ on $H$ satisfies 
    \begin{eqnarray}\label{eq: sub gaussian prior}
        \E^{\mu} \exp(C_1 \norm{x}^2) < \infty.
    \end{eqnarray}
    for some $C_1>0$.
\end{assumption}

\begin{assumption}\label{ass: assumption}
     Let $\F$ be a convex, closed set of non-negative functions and assume that the true potential $L$ is contained in $\F$, that is $L \in \F$. Assume also that there exits $C_2^-, C_2^+ > 0$ and for every $\varepsilon > 0$ there exists $\delta >0$ such that for every $f \in \F$ and all $x \in H, y \in \R^m$
        \begin{eqnarray}\label{eq: bounds on f}
             -C_2^- - \varepsilon \norm{x}^2   +\delta \norm{y}^2 \le f(x, y) \le C_2^+  (1 + \norm{x}^2   +\norm{y}^2).
        \end{eqnarray}
\end{assumption}
The assumption \eqref{eq: bounds on f} is stronger than non-negativity of $f$. It will ensure that the marginal distribution of $y$ under $f$ is well defined (see Equation \eqref{eq: marginal of y corresponding to L} and Lemma~\ref{lem: Z_fy bounds}).
The following examples give convenient sets of sufficient conditions under which Assumption~\ref{ass: assumption} holds for Gaussian likelihood potentials.

\begin{example}\label{ex: gaussian likelihood 1}
     Let $\mathcal{A}$ be a given set of potential forward maps, and let the noise covariance $\Gamma \in \R^{m \times m}$ be such that the following hold uniformly for all $A \in \mathcal{A}$:
    \begin{itemize}
        \item[(B1)] (Lipschitz) There exists $C_3 > 0$ such that for all $x, x' \in H$
        \begin{eqnarray*}
            \norm{A(x) - A(x')}_\Gamma  \le C_3 \norm{x - x'}.
        \end{eqnarray*}
        \item[(B2)] ($A$ is proper) There exist $R, C_4 > 0$ such that
        \begin{eqnarray*}
            \mu\left( B(0, R)\right) > 0 \text{ and } \sup\limits_{x \in B(0, R)} \norm{A(x)}_\Gamma < C_4.
        \end{eqnarray*}
    \end{itemize}
    Then by \cite[Lemma 3.10]{helin2025bayesianoptimalexperimentaldesign}, the class 
    \begin{eqnarray*}
        \F_{\text{residual}} := \left\{f: (x, y) \mapsto \frac{1}{2} \norm{A(x) -y }_\Gamma^2 \bigg| A \in \mathcal{A} \right\}
    \end{eqnarray*}
     satisfies Assumption \ref{ass: assumption}.
\end{example}
\begin{example}\label{cor: assumptions}
    Let $\mathcal{A}$ be a given set of potential forward maps that satisfies (B1) and (B2) in Lemma~\ref{ex: gaussian likelihood 1} and $\mathcal{G}$ a set of potential covariances such that 
    \begin{itemize}
        \item[(B3)] (Bounded eigenvalues) The eigenvalues of $\Gamma \in \mathcal{G}$ are all uniformly bounded from below and above, that is there exist $0 < C_5^-, C_5^+$ such that for all $\Gamma \in \mathcal{G}$\begin{eqnarray*}
        C_5^- \le \gamma_{min}(\Gamma) \le \gamma_{max}(\Gamma) \le C_5^+,
    \end{eqnarray*}
    where $\gamma_{min}(\Gamma)$ and $\gamma_{max}(\Gamma)$ denote the minimum and maximum eigenvalues of $\Gamma$ respectively.
    \end{itemize}
    Then the class
        \begin{eqnarray*}
        \F_{\text{calibrated}} := \left\{f: (x, y) \mapsto \frac{1}{2} \norm{A(x) -y }_\Gamma^2 \bigg| A \in \mathcal{A}, \Gamma \in \mathcal{G}\right\}
    \end{eqnarray*}
    is contained in a set that satisfies Assumption \ref{ass: assumption}.
\end{example}
\subsection{An equivalence class for likelihood functions}
\begin{definition}\label{def: approximate posterior}
    To any $f \in \F$, we assign a posterior distribution $\mu^y_f$ according to
\begin{eqnarray*}
    \frac{d \mu^y_f}{d \mu}(x) = \frac{1}{Z_f(y)} \exp \left( -f(\cdot; y) \right) \text{ in } L^1(\mu)
\end{eqnarray*}
 where
\begin{eqnarray}\label{eq:normalization}
        Z_f(y) := \int_H \exp \left( -f(x; y) \right) \mu(dx)
\end{eqnarray}
with $\mu$-almost everywhere in $x$.
We define the marginal distribution $\pi_f(\cdot)$ of $y$ through
 \begin{eqnarray}\label{eq: marginal of y corresponding to L}
     \pi_f(y) := \frac{Z_f(y)}{\int_{\R^m} Z_f(y) dy}.
 \end{eqnarray}
 and the joint distribution $\lambda_f(dx, dy)$ by
 \begin{eqnarray}\label{eq: joint}
     \lambda_f(dx, dy) = \pi_f(dy)\frac{e^{-f(x, y)}}{Z_f(y)}\mu(dx).
 \end{eqnarray}
\end{definition}
According to Lemma~\ref{lem: Z_fy bounds}, $\infty > \int Z_f(y) dy > 0$, so Equation~\eqref{eq: marginal of y corresponding to L} is well-defined.
The likelihoods $\exp(-L)$ or $\exp(-f)$ do in general not integrate to one. Instead normalization is enforced by formally dividing $Z(y)$, so any functions that differ by an additive constant will give rise to the same posterior distribution.
We define now an equivalence class of functions $\F_\Phi$, where $f \sim g$ for $f, g \in \F$ if for all $x, y \in H\times \R^m$
\begin{eqnarray*}
   f(x, y) + \log Z_f(y) = g(x, y) + \log Z_g(y).
\end{eqnarray*}
We equip $\F_\Phi$ with the metric
\begin{eqnarray}
\begin{aligned}\label{eq: norm for equivalence}
        \norm{f -g }_\Phi &:= \norm{f(\cdot, \cdot) -g(\cdot, \cdot)  + \log Z_f(\cdot)  - \log Z_g(\cdot) }_{L^1(H \times \R^m, \lambda)}.
\end{aligned}
\end{eqnarray}
Notice that $\Phi$ and $\mu^y_f$ and $\lambda_f$ are indeed independent of the choice of the representative in $\F_\Phi$. We shall view the map $\Phi$ as a mapping of an equivalence class $\F_\Phi \rightarrow \R$ from now on.
\begin{lemma}\label{lem: equivalence class}
    For $f, g \in \F_\Phi$, it holds that $f \sim g$ iff $f-g$ depends only on $y$ on the support of $\lambda$.
\end{lemma}
\begin{proof}[Proof of Lemma~\ref{lem: equivalence class}]
        Assume $f\sim g$. Then for any $(x, y) \in \mathrm{supp}(\lambda)$
        \begin{eqnarray*}
            f(x, y) - g(x, y) = \log Z_g(y) - \log Z_f(y).
        \end{eqnarray*}
        Since the right hand-side depends only on $y$ so does the left hand-side. Suppose now that $f = g + h$, where $h$ depends only on $y$ on $\mathrm{supp}(\lambda)$. Then
        \begin{eqnarray*}
            f(x, y) + \log Z_f(y) &=& g(x, y) + h(y) + \log \int_H \exp(-g(x, y) - h(y)) \mu(dx) \\
            & = & g(x, y) + \log \int_H \exp(-g(x, y)) \mu(dx),
        \end{eqnarray*}
        that is $f\sim g$.
    \end{proof}

The following lemmas ensure that $\norm{f - g}_\Phi$ is finite for  $f, g \in \F$ and will be used later. 
  \begin{lemma}\label{lem: subgaussian joint}
    There exists universal constants $C, \kappa>0$ such that for every $f \in \F$
    \begin{eqnarray*}
       \E^{\lambda_f} e^{\kappa (\norm{x}^2 + \norm{y}^2)}   < C
    \end{eqnarray*}
\end{lemma}

 \begin{lemma}\label{lem: lambda bounded by Lp}
    There exists $p_0>1$ such that for all $p\ge p_0 $ and fixed $h \in \F$, there exists a universal constant $ C > 0$ and for any $f, g \in \F$ it holds
    \begin{eqnarray*}
        \E^{\pi_h} \left| \log Z_f(y) - \log Z_g(y) \right| \le C \norm{f-g}_{L^p\left(\lambda_h(\cdot, \cdot)\right)}
    \end{eqnarray*}
\end{lemma}
By Lemma \ref{lem: lambda bounded by Lp} the metric $\norm{f-g}_\Phi$ is dominated by $\norm{f-g}_{L_p(\lambda)}$. Since $\lambda$ is sub-Gaussian by Lemma, \ref{lem: subgaussian joint}, the norm $\norm{f-g}_{L_p(\lambda)}$ is finite. Bounding the expected difference of the normalization terms $\log Z_f(y)$ will be crucial in the following.

\subsection{Main properties}

The first observation that we make, is that if we have access to the joint distribution $\lambda$ we can compute the KL-divergence in \eqref{eq: objective} without knowledge of the mathematical expression for $L$, up to additive constants (independent of $f$). This observation is used in practice for neural likelihood approximation in parametric classes or neural flows (see \cite{papamakarios2019sequantialneural, radev2020bayesflowlearningcomplexstochastic}).
\begin{lemma}\label{lem: basic lemma}
    Let Assumptions  \ref{ass: prior} and \ref{ass: assumption} hold. Then 
    \begin{eqnarray*}
         \E^\pi\KL{\mu^y}{\mu^y_f}
     =  \E^\pi \KL{\mu^y}{\mu} + \E^{\lambda}[f(x;y)] + \E^\pi[\log Z_f(y)]< \infty.
    \end{eqnarray*}
\end{lemma}
\begin{proof}[Proof of Lemma~\ref{lem: basic lemma}]
The decomposition is immediate from
    \begin{eqnarray*}
         \E^\pi \, \E^{\mu^y} \left[\log\left( \frac{d \mu^y}{d \mu^y_f}(x) \right) \right]  = \E^\lambda \left[\log\left( \frac{d \mu^y}{d \mu^y_f}(x) \right) \right]
         = \E^\pi \KL{\mu^y}{\mu} - \E^\lambda \left[\log \frac{d \mu^y_f}{d\mu} \right].
    \end{eqnarray*}
    We show that all quantities are finite. Since by Lemma~\ref{lem: subgaussian joint}, $\lambda = \lambda_L$ is sub-Gaussian and for $f \in \F$, it holds
   $\E^\lambda f(x, y) <  \infty$.
    By using the lower bound in \eqref{eq: bounds on f} together with \eqref{eq: sub gaussian prior},
    \begin{eqnarray*}
         \E^\pi \log \E^{\mu}\exp(-f(x, y))
        & \le & \E^\pi \log \left(\E^\mu \exp\left(C_2^- + C_1 \norm{x}^2 - \delta(C_1) \norm{y}^2 \right) \right) \\
        & = & C_2^- - \delta(C_1) \, \E^\pi \left[\norm{y}^2\right] + \log \left(\E^\mu \exp(C_1 \norm{x}^2) \right) < \infty.
    \end{eqnarray*}
    Finally, by using $L \ge 0$, the upper bound in \eqref{eq: bounds on f} and Jensen's inequality to the concave function $x\mapsto \log x$
    \begin{eqnarray*}
        \E^\pi \KL{\mu^y}{\mu} & = & -\E^\lambda  \left[L(x, y)\right] - \E^\pi \log \left[  \E^\mu \exp\left( -L(x, y)\right) \right]  \\
        & \le  & - \E^\pi \log \left[\E^\mu \exp\left(-C_2^+ -C_2^+ \norm{x}^2-C_2^+ \norm{y}^2\right)\right] \\
        & \le & C_2^+  (1+\E^\pi \norm{y}^2)  + C_2^+ \, \E^\mu \norm{x}^2. 
    \end{eqnarray*}
\end{proof}

Motivated by Lemma~\ref{lem: basic lemma}, we define an objective function $\Phi$ as
\begin{eqnarray}
\begin{aligned}\label{eq: Phi equation}
    &  \Phi: \F_\Phi \rightarrow \R,  \quad f \mapsto \E^\lambda [f(x;y)] + \E^\pi [\log Z_f(y)],
\end{aligned}
\end{eqnarray}
where $Z_f(y)$ is given in Equation~\eqref{eq:normalization}.
The following result is motivated by an analogous result in \cite{Helin2025OAE} for the covariance in a Gaussian likelihood with fixed mean. We now show that $\Phi$ is convex as a function of the NLL.
\begin{theorem}\label{thm: convex}
    Let Assumptions  \ref{ass: prior} and \ref{ass: assumption} hold. Then the function $\Phi$ given by \eqref{eq: Phi equation} is strictly convex.
\end{theorem}

\begin{proof}[Proof of Theorem~\ref{thm: convex}]
It holds for any $\alpha + \beta = 1, \alpha, \beta > 0$ and $f, g \in \F$ that
    \begin{eqnarray*}
         \Phi(\alpha f + \beta g) &=& \alpha \, \E^\lambda f(x, y)  + \beta \, \E^\lambda g(x, y)   + \E^\pi \log \E^\mu  \left[\exp \left( -f(x, y)  \right)^\alpha  \exp \left( - g(x, y)  \right)^\beta \right] \\
         & \le & \alpha \, \E^\lambda f(x, y)  + \beta \, \E^\lambda g(x, y)  \\
         & & + \alpha \, \E^\pi \left[\log \E^\mu  \exp \left( -f(x, y)  \right) \right] + \beta \, \E^\pi \left[\log \E^\mu  \exp \left( -g(x, y)  \right) \right]\\ 
         & =& \alpha \, \Phi(f) + \beta \, \Phi(g),
    \end{eqnarray*}
    where Hölder-inequality was used. We notice further that Hölder-inequality is an equality iff there exists $c(y)>0$ such that $\mu$-almost everywhere
    \begin{eqnarray*}
        \exp \left( -f(x, y)\right) = \exp \left( - g(x, y)  \right)c(y),
    \end{eqnarray*}
    i.e. $f = g$ in $\Phi_\F$.
\end{proof}

\begin{remark}
The convexity of $\Phi$ is a property of the objective as a functional of the
negative log-likelihood $f$, rather than of a particular parametrization.
Consequently, if $f$ is represented by a neural network, the optimization
problem in the network parameters remains non-convex. Nevertheless, strict
convexity implies that the population objective admits a unique minimizer in
$\mathcal{F}_\Phi$, providing an identifiability result that underlies the
consistency analysis in Section~\ref{sec: data driven}. Different
parameterizations may therefore have many local minima in parameter space while
still representing the same optimal likelihood function.
\end{remark}

\section{Consistency and data-driven methods}\label{sec: data driven}
In this section, we consider an approximation of the integral in the objective function with (finite) data. Our goal will be to show the consistency of the NLL approximation.
To this end, we recall the definition of consistency and a classical theorem from \cite{Vaart_1998}. 
For a given domain $\D$, let $\phi_N: \D \rightarrow \R$ be random functions and $\phi: \D \rightarrow \R$ a fixed function, denote by $f_0 := \argmin_{f \in \D} \phi(f)$ and let $d$ be a metric on $\D$. We say a sequence of estimators $\hat{f}_N$ is \emph{(asymptotically) consistent}, if $\hat f_N$ converges to $f_0$ in probability.
\begin{theorem}[Consistency; \cite{Vaart_1998}]\label{thm: consistency, Vaart}
    Assume there exists $f_0 \in \D$ such that for every $\varepsilon > 0$
    \begin{align}
        \inf_{d(f, f_0) > \varepsilon}  \phi (f) >& \phi(f_0). \label{eq: uniqueness condition}
    \end{align}
    and 
    \begin{align}
                \sup_{f \in \D}  \big| \phi_N(f) -& \phi(f) \big|  \rightarrow 0 \label{eq: glivenko cantelli condition} 
    \end{align}
    in probability.
    Then any $\hat f_N$ satisfying $\phi_N(\hat f_N) \le \inf_{f\in \D} \phi_N(f) + o_P(1)$ converges in probability to $f_0$, where $o_P(1)$ converges to zero in probability.
\end{theorem}
The addition of the term $o_P(1)$ in the preceding theorem allows for incomplete minimization of the objective. For example one could choose the term $1/N$, where $N$ is the data size and stop minimization once in a $1/N$ neighbourhood of the minimum.

Our goal is to use the preceding theorem for a suitable data-driven estimator $\Phi_N$ of $\Phi$. To this end, we first make some necessary assumptions.

\subsection{Assumptions}
\begin{definition}
     For $s >0$ integer and $\beta \in \R$, let $\F$ be a bounded subset of 
    \begin{eqnarray*}
        C^s(\R^{n + m}, \beta) := \left\{ f: f (1+ \norm{z}^2)^{\beta/2} \in C^s(\R^{n+m}) \right\},
    \end{eqnarray*}
    where $C^s(\R^{n+m}) $ is the space for which
    \begin{eqnarray*}
        \norm{f}_{s, \infty} := \sum\limits_{0 \le |\alpha| \le s} \norm{D^\alpha f}_\infty < \infty
    \end{eqnarray*}
    is finite. 
\end{definition}

\begin{assumption}\label{ass: finite dimension}
    Fix the discretization $H = \R^n$ for some $\infty > n \in \Natural$ and let $\F$ be a bounded subset of functions in $C^1(\R^n \times \R^m, -2)$.
\end{assumption}

\begin{example}
Assume that the assumptions of Example~\ref{cor: assumptions} hold for $\mathcal{A}$ and $\mathcal{G}$, let $H = \R^n$ and
\begin{itemize}
    \item[(B4)] (Differentiability) Any $A\in \mathcal{A}$ is differentiable and there exists $C_6>0$ such that
        \begin{eqnarray*}
            \norm{\nabla_x A(x)} \le C_6 \norm{x}.
        \end{eqnarray*}
\end{itemize}
Then the set 
    \begin{eqnarray*}
        \F_{\text{calibrated}} =\left\{f: (x, y) \mapsto \frac{1}{2} \norm{A(x) -y }_\Gamma^2 \bigg| A \in \mathcal{A}, \Gamma \in \mathcal{G}\right\}
    \end{eqnarray*}
    is contained in a set of function that satisfies Assumptions~\ref{ass: assumption} and \ref{ass: finite dimension}.
\end{example}

\subsection{Data-driven methods}
Now we are ready to present a data–driven estimator $\Phi_N$ of $\Phi$. To this end, we assume that the acquisition of joint data is restricted by computational cost or physical constraints and we can only acquire a limited ensemble $(x_i, y_i)_{i = 1}^{N} \sim \lambda$ of $N$ independent, identically distributed (iid) (supervised) training data. We further assume that data from the prior is obtained easily and in abundance so for each $y_i$, we can get an iid (unsupervised) training data $(x_{i, j})_{j =1}^{M(N)} \sim \mu$ of size $M(N)$ for a coercive function $M: \Natural \to \Natural$.
We introduce a double loop Monte-Carlo estimator for the term $\Lambda(f):= \E^\pi \log \E^\mu e^{-f(x, y)}$, which is approximating the expectation with respect to the marginal $\pi$ of $y$ and the prior $\mu$, by
\begin{eqnarray*}
    \Lambda_{N}(f) := \frac{1}{N} \sum_{i=1}^N \log Z_f^{M(N)}(y_i), \qquad Z_f^M(y_i) := \frac{1}{M} \sum_{j=1}^{M} e^{-f(x_{i, j}, y_i)}.
\end{eqnarray*}
An empirical estimator $\Phi_N$ of $\Phi$ is given by
\begin{eqnarray*}
    && \Phi_N: \F_\Phi \rightarrow \R, \quad
    f   \mapsto \frac{1}{N}\sum_{i = 1}^{N} f(x_i, y_i)  + \Lambda_N(f),
\end{eqnarray*}  
Almost sure convergence of the first summand is straight forward from the law of large numbers. Convergence properties of nested Monte Carlo  are a lot more involved (see the discussions in \cite{rainforth18onnestiting, bartuska2025doublelooprandomizedquasimontecarlo}). In particular, we leave the choice of an optimal function $M$ open.
The following theorem establishes convergence in probability of the nested Monte-Carlo estimator.
\begin{theorem}\label{lem: conv in p}
    Let $f \in \F$ and Assumptions~\ref{ass: prior} and~\ref{ass: assumption} hold. Further suppose $(x_i, y_i)_{i = 1}^{N} \sim \lambda$ and $((x_{i, j})_{j=1}^{M(N)})_{i =1}^{N}$ be iid. Then the convergence
    \begin{eqnarray*}
        \lim_{N \rightarrow \infty}\Lambda_{N}(f) = \Lambda(f)
    \end{eqnarray*}
    holds in probability.
\end{theorem}

\subsection{Consistency}
Our goal in this section is to establish consistency of a data-driven
approximate minimizer $f_N^*$ towards the true potential $L$. To do so we
verify the two hypotheses of Theorem~\ref{thm: consistency, Vaart} for
$\phi = \Phi$, $\phi_N = \Phi_N$ and $d = \norm{\cdot}_\Phi$.  
In Theorem~\ref{lem: uniqueness cond}, we verify the condition of a well-separated minimum \eqref{eq: uniqueness condition} and in Theorem~\ref{thm: M estimation} a uniform law of large numbers \eqref{eq: glivenko cantelli condition} for the empirical objective
$\Phi_N$. Combining the two then yields
the consistency result, Theorem~\ref{thm: M estimation ours}, which is the
main result of this section.

Since existence of an exact minimizer of $\Phi_N$ over $\F_\Phi$ cannot be
guaranteed in general, we work throughout with approximate minimizers $f_N^*$ satisfying
\begin{eqnarray}\label{eq: f_nstar}
    \Phi_N(f_N^*) \le \inf\limits_{f \in \F_{\Phi} } \Phi_N(f) +o_P(1),
\end{eqnarray}
where $o_P(1) \to 0$ in probability as $N \to \infty$ (recall from the
discussion after Theorem~\ref{thm: consistency, Vaart} that this allows,
for instance, minimization only up to tolerance $1/N$).
\begin{theorem}\label{lem: uniqueness cond}
Let Assumptions~\ref{ass: prior},~\ref{ass: assumption} and~\ref{ass: finite dimension} hold. Then for any $f \in \F_\Phi$
    \begin{eqnarray*}
        \inf\limits_{\norm{f-L}_\Phi > \varepsilon} \Phi(f) > \Phi(L).
    \end{eqnarray*}
\end{theorem}
\begin{theorem}\label{thm: M estimation}
    Let Assumptions~\ref{ass: prior},~\ref{ass: assumption} and~\ref{ass: finite dimension} hold.
    Then
    \begin{eqnarray*}
         \sup\limits_{f \in \F} |
         \Phi_N(f) - \Phi(f) | \rightarrow 0
    \end{eqnarray*}
    in probability.
\end{theorem}
Theorems~\ref{lem: uniqueness cond} and~\ref{thm: M estimation} verify
conditions \eqref{eq: uniqueness condition} and
\eqref{eq: glivenko cantelli condition} of
Theorem~\ref{thm: consistency, Vaart} with $f_0 = L$ (which is indeed the
unique minimizer of $\Phi$ over $\F_\Phi$, by
Lemma~\ref{lem: basic lemma} together with
Lemma~\ref{lem: equivalence class}). Theorem~\ref{thm: consistency, Vaart}
then applies directly and yields the following consistency guarantee for
the approximate minimizer $f_N^*$ defined in \eqref{eq: f_nstar}.
\begin{theorem}\label{thm: M estimation ours}
    Let $f_N^*$ be given by \eqref{eq: f_nstar} and let Assumptions~\ref{ass: prior},~\ref{ass: assumption} and~\ref{ass: finite dimension} hold. Then for every $\varepsilon > 0$
    \begin{eqnarray*}
        \mathbb{P}(\norm{f_N^* - L}_\Phi > \varepsilon) \rightarrow 0  \text{ as } N \rightarrow \infty,
    \end{eqnarray*}
    where the symbol $\mathbb{P}(A)$ denotes the probability of a measurable event $A$ under the probability measure $\lambda$.
\end{theorem}
To prove Theorem~\ref{thm: M estimation}, we need the following results on bracketing numbers.
\begin{definition}
    An $\varepsilon$-bracket $[l, u]$ in $L^p(\lambda)$ with $p\ge 1$ with functions $l \le u$ is defined as the set of all functions $f$ such that $l \le f \le u$ and $\int (u-l)^p d \lambda \le \varepsilon^p$. 
    The bracketing number $N_{[]}(\varepsilon, \F, L^p(\lambda))$ is the minimum number of $\varepsilon$-brackets needed to cover $\F$.
\end{definition}
We recall the following powerful result in \cite[Corollary 3]{NicklBracketing2007} on bracketing numbers for bounded subsets of Hölder spaces.
\begin{theorem}\label{thm: bracketing number}
   Let $s >0$ be integer, $p\ge 1$ and $\beta \in \R$. Let $\F$ be a bounded subset of $C^s(\R^{n+m}, \beta)$ and 
   \begin{eqnarray}\label{eq: moments for brackets}
       \norm{(1+ \norm{x}^2)^{(\gamma - \beta)/2}}_{p, \lambda}  < \infty
   \end{eqnarray}
   for some $\gamma > 0; \gamma \neq s$.
    Then 
    \begin{eqnarray*}
        \log N_{[]}(\varepsilon, \F, L_{p}(\lambda)) < C \varepsilon^{-(n+m)/\min({s, \gamma})}. 
    \end{eqnarray*}
    In particular, if Assumptions~\ref{ass: finite dimension} are satisfied, the bracketing number $N_{[]}(\varepsilon, \F, L_{p}(\lambda))$ is finite for every $p \ge 1$. (Choose any $\gamma>0$. Since $\lambda$ is sub-expontinental, the moments in  \eqref{eq: moments for brackets} are finite.)
\end{theorem}


\begin{remark}
It is difficult to apply classical rate-of-convergence results for M-estimators
\cite{van1996weak} to our setting.
The difficulty is twofold. First, the second term in $\Phi_N$
contains the nested Monte Carlo estimator $\Lambda_N(f)$, which is
not a linear empirical average. Although
Theorem~\ref{lem: conv in p} proves convergence of $\Lambda_N(f)$,
it does not provide quantitative or uniform rates over
$\mathcal{F}_\Phi$. Second, the classical theory is based on
localized empirical processes controlled by $L^2(\lambda)$
bracketing entropy, whereas our analysis relies on
$L^p(\lambda)$ estimates (Lemma~\ref{lem: lambda bounded by Lp}) for
possibly large $p$. Extending the existing theory to this setting is
left for future work.
\end{remark}

\section{Numerical demonstration}\label{sec: numerics}
In this section we consider the deblurring problem and the problem of recovering the doping profile in a semiconductor. Our goal is to test the free-form, residual and calibrated residual approximations of the negative-log-likelihood as given in Section~\ref{s:NLL approximation}. In the first example, we further compare with  a GP regression of the likelihood. The GP assumes a squared exponential covariance with learnable length-scale $l\in \R$ and variance level $\sigma^2 \in \R$. To emphasize the learnable parameters, we denote by $f^\theta: \R^n\to \R$ or $g^\theta:\R^n \to \R^m$ a neural-network with learnable parameters $\theta$ and by $\Sigma_\theta \in \R^{m\times m}$ a learnable covariance matrix. The underlying output functions and assumptions made are summarized in Table~\ref{tab: NLL appproxs}. 

\begin{table}[h]
    \centering
        \caption{Parametrization and required data for the NLL approximations and the GP regression. Training the free-form and calibrated residual-approximation requires only joint data, residual approximation needs additionally the noise level and GP regression needs access to evaluations of the true NLL.}
    \label{tab: NLL appproxs}
    \begin{tabular}{c | c c c c}
         & Free-form & calibrated residual-  & residual- & GP  \\
         & approximation& approximation  & approximation& regression \\
         \hline
        output &  $f^\theta(x, y)$ & $\frac{1}{2} \norm{g^\theta(x) - y}^2_{\Sigma_\theta}$ & $\frac{1}{2 \delta^2} 
        \norm{g^\theta(x) - y}^2$ & GP(x, y) \\
        \hline 
        required & $(x, y)$- data & $(x, y)$- data  & $(x, y)$- data &  $(x, y)$- data \\
        data & & & and $\delta^2$ & and $L(x, y)$ \\
        \hline 
        learnable & $f^\theta$ & $g^\theta, \Sigma_\theta$ & $g^\theta$ &  $\sigma^2, l$ \\
        parameters &  &  & & 
    \end{tabular}
\end{table}
\subsection{Deblurring}
We consider a de-blurring problem as the first example for the numerical experiments. Define  the corresponding Bayesian inverse problem as following:
The unknown $x \in H$ is supported on the Hilbert space $H = L^2(D)$ with $D = [0, 1]$.
The prior is given by a GP $\N(0, C)$ with squared exponential kernel $C$:
\begin{eqnarray*}
    C(x, y) = \frac{1}{\sigma^2}\exp\left(-\frac{1}{2l^2}\norm{x-y}^2 \right), \quad x, y \in D, 
\end{eqnarray*}
where the parameter $\sigma = 0.1$ models the variance and $l =0.1$ the length-scale of the process.
The forward map $A: L^2(D) \rightarrow \R^m$ blurs any given unknown $x \in L^2(H)$ by computing the convolution
\begin{eqnarray*}
    \left( x \ast k \right)(t) = \int_{[0, 1]} x(s) k(t-s) \, ds
\end{eqnarray*}
for a given kernel function $k$. This convolution is then evaluated at a finite grid $(t_1, \cdots, t_d) \subset \R^m$ with $m = 50$:
\begin{eqnarray*}
    Ax=\begin{pmatrix}
    (x * k)(t_1) \\
    \vdots \\
    (x * k)(t_d)
    \end{pmatrix}\in \mathbb{R}^d .
\end{eqnarray*}
We choose $k$ to be the PDF of a Gaussian centred random variable with variance two. The measurement $y \in \R^m$ is obtained by corrupting $Ax$ with additive noise:
\begin{eqnarray*}
    y = Ax + \varepsilon,
\end{eqnarray*}
where observational noise $\varepsilon \sim \N(0, \delta^2 I_d)$ is Gaussian with white noise statistics and noise level $\delta = 0.1$. 
The free-form NLL approximation uses $4$ fully connected layers, the residual approximation $5$ fully connected layers and the calibrated residual approximation uses $5$ fully connected layers followed by a diagonal layer that approximates the noise level.

We deploy a Markov Chain Monte Carlo (MCMC) algorithm, namely the preconditioned Crank-Nicolson (pCN) \cite{NicklRichard2023Bnsi, Taghizadeh2025Bayesian} algorithm to sample from the posterior distributions. In the pCN algorithm, a proposal $v_n$ given a previous sample $u_n$ are generated according to
\begin{equation*}
    v_n = \sqrt{1-2\beta} u_n + \sqrt{2 \beta} \zeta \qquad \zeta \sim \mu,
\end{equation*}
where $\mu$ is a Gaussian prior and $\beta \in (0, 1)$ a hyperparameter. We choose $\beta = 0.2$ in this Section. The proposal $v_n =: u_{n+1}$ is then accepted with probability 
\begin{eqnarray*}
    p_n = \min\left(\exp(L(v_n, y)-L(u_n, y), 1 \right),
\end{eqnarray*}
where $L$ is the NLL.
In Figure~\ref{fig: NLL mcmc_for_deblurring}, we first analyse the different NLL approximations. We demonstrate that the posterior mean of the free-form, residual and calibrated residual approximations are very close to the ground truth and contained in a given confidence interval. The free-form approximation is slightly closer to the ground truth with slightly smaller confidence band compared to the residual approximation and calibrated residual approximations. The approximations have slightly lower accuracy than using the true posterior and wider confidence bands. The calibrated residual approximation is furthest from the ground truth.
\begin{figure}[h]
    \centering
    \begin{minipage}{0.23\textwidth}
        \centering
        Exact posterior \\
        \phantom{...} \\
        \includegraphics[width=\linewidth]{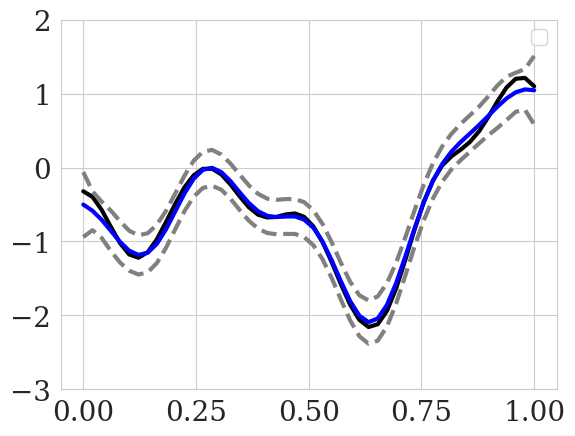}
    \end{minipage}%
    \hfill
    \begin{minipage}{0.23\textwidth}
        \centering
        Free-form approximation \\
        \includegraphics[width=\linewidth]{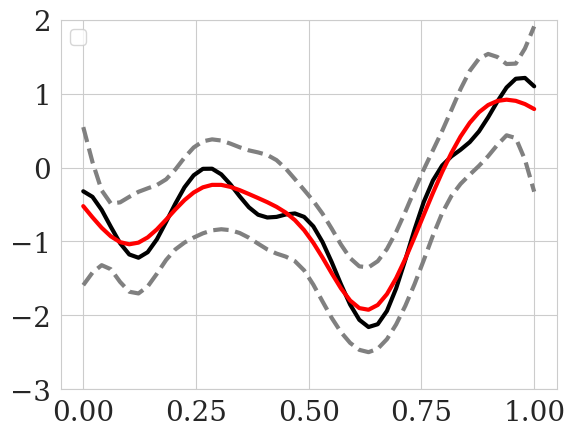}
    \end{minipage} 
    \hfill
        \begin{minipage}{0.23\textwidth}
        \centering
        Residual approximation \\
        \includegraphics[width=\linewidth]{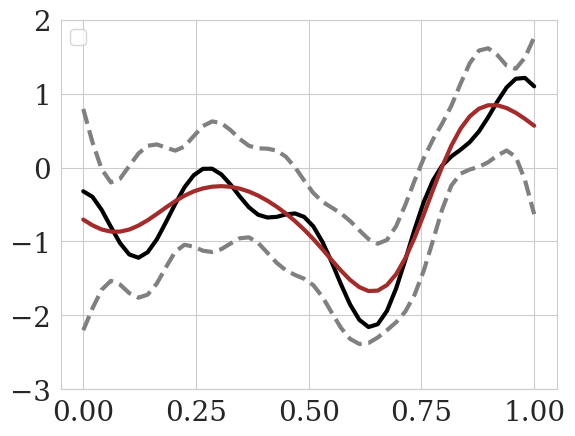}
    \end{minipage}%
        \hfill
        \begin{minipage}{0.23\textwidth}
        \centering
        Calibrated residual approximation \\
        \includegraphics[width=\linewidth]{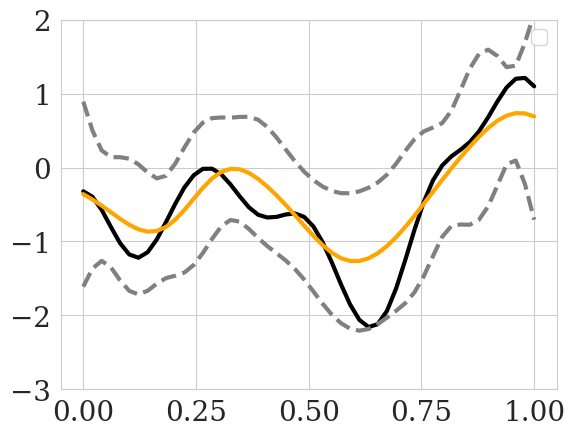}
    \end{minipage}%
    \caption{Posterior MCMC samples for the deblurring problem. From left to right: true posterior, free-form approximation, residual approximation and calibrated residual approximation. All methods use only $1,000$ training points.}
    \label{fig: NLL mcmc_for_deblurring}
\end{figure}

In Figure \ref{fig: mcmc_for_deblurring}, we illustrate that GP regression, even though it utilizes the additional information of evaluations of the true likelihood, struggles to accurately describe the likelihood especially with 1,000 training points. Even with 10,000 training points, it is less accurate than all of the NLL approximations.
\begin{figure}[h]
    \centering
    \begin{minipage}{0.3\textwidth}
        \centering
        Exact posterior \\
        \phantom{...} \\
        \includegraphics[width=\linewidth]{images/mcmc_exact.png}
    \end{minipage}%
        \hfill
        \begin{minipage}{0.3\textwidth}
        GP regression \\
        1,000 training points \\
        \centering
        \includegraphics[width=\linewidth]{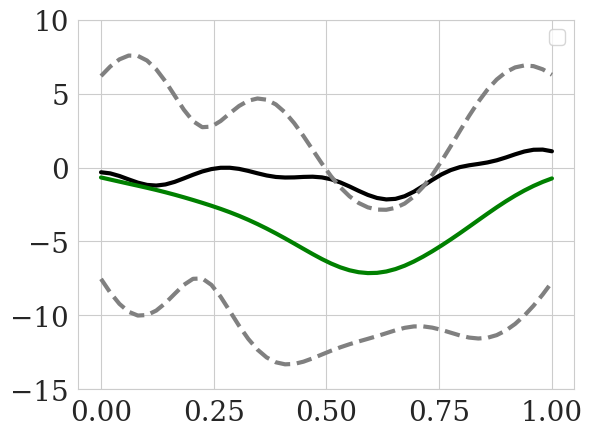}
    \end{minipage}%
    \hfill
    \begin{minipage}{0.3\textwidth}
        \centering
        GP regression \\
        10,000 training points \\
        \includegraphics[width=\linewidth]{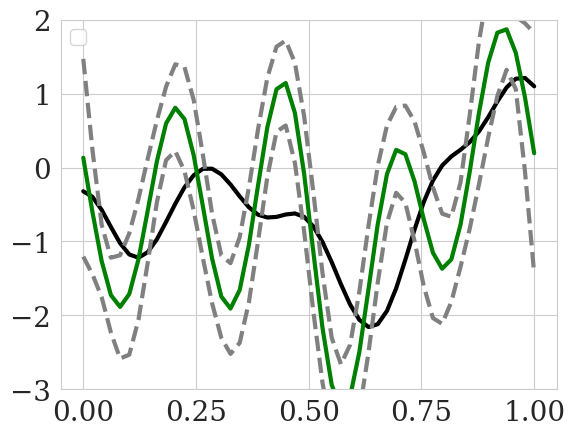}
    \end{minipage}%
    \caption{Posterior MCMC samples for the deblurring problem. From left to right: true posterior, GP regression with 1,000 training points and GP regression for the likelihood with 10,000 training points.}
    \label{fig: mcmc_for_deblurring}
\end{figure}
This experiment demonstrates that neural likelihood approximation is a highly data-efficient way to accurately approximate the likelihood in inverse problems. In settings where the observational noise is unknown, the free-form approximation seems preferable over a calibrated residual approximation based on this experiment.

\FloatBarrier
\subsection{Semiconductor devices}

A non-invasive method for estimation of the doping profile inside a semiconductor is the identification by voltage-current measurements. For a thorough introduction we refer to \cite{Taghizadeh2025Bayesian} and the references therein.
In short, the domain $\Omega = [-1, 1]^2$ models the cross-section of a $pn$-diode with Ohmic contacts $\Gamma_N = [-1, 1] \times\{ +1\}$ and $\Gamma_P = [-1, 1] \times\{ -1\}$. $\Omega$ is separated into the sets $\Omega_N$ and $\Omega_P$ by a known junction line $\Gamma$.
We assume that the unknown function $C$ is given by two independent Gaussian random fields $C_N, C_P$ separated by the junction line $\Gamma$:
\begin{eqnarray}\label{eq: piecewise constant C}
    C(x) = \begin{cases}
        C_N(x) & x \in \Omega_N \\
        C_p(x) & x \in \Omega_p.
    \end{cases}
\end{eqnarray}

\paragraph{Forward model}
We consider the following model based on the (close to) thermal equilibrium for  the electric potential $V_e$ given a concentration $C(x)$
\begin{equation}
    \begin{aligned}\label{eq: poisson equation}
    \lambda^2 \triangle V_e &= \delta^2 \left( e^{V_e} - e^{-V_e}\right) - C \quad  \text{ in } \Omega \\
    V_e &= V_{bi} \quad \text{ on } \partial \Omega_D \\
    \nabla V_e \cdot \nu & = 0 \quad \text{ on } \Omega_N.
    \end{aligned}
\end{equation}

Given $V_e$, we observe the quantity $\delta^2 e^{V_{bi}}\nabla \hat u \cdot \nu |_{\Gamma_P}$, where $\hat u$ solves the continuity equation
\begin{equation}\label{eq: continuity equation}
    \begin{aligned}
    \div\left( \gamma(x) \nabla \hat u(x) \right) & = 0 \text{ in } \Omega \\
    \hat u & = U \text{ on } \partial \Omega_D \\
    \nabla \hat u \cdot \nu & = 0 \text{ on } \partial \Omega_N
    \end{aligned}
\end{equation}
and $\gamma = e^{V_e}$. We fix the parameters $V_{bi} = 0.6$, $\lambda = \delta = 1$ and $U = 2$.

\paragraph{Prior model}
We model the spatial functions $C_N$ and $C_p$ as independent Gaussian fields with Matern–Whittle covariance operator given 
\begin{eqnarray*}
    K(x, y) = \sigma^2 \frac{2^{1-\nu}}{\Gamma(\nu)} \left(\frac{\norm{x-y}_2}{l}\right)^\nu K_\nu\left(\frac{\norm{x-y}_2}{l}\right),
\end{eqnarray*}
where $\Gamma$ denotes the Gamma function, $K_\nu$ the modified Bessel function and the parameters $\sigma, \nu, l$ denote the variance, smoothness and length scale of the process. We choose the parameters $\sigma = 2, l = 20$ and $\nu = 0.9$. 

\paragraph{Training}
Training data is given by $N = 10,000$ observations from the joint distribution of $\lambda$ with a resolution of $100 \times 100$ for $C$ and measurement resolved at $100$ pixels. The PDE is solved using a finite element solver in MATLAB and with a mesh-size of $0.1$. Generating the training data takes $\approx 5$ hours.
The boundary measurement is corrupted with additive noise of level $\sigma = 0.01$. We assume that the true forward model is unknown and test the free-form, residual and calibrated residual approximations, all of which can be employed without knowing the forward map. We use similar number of parameters for the underlying neural networks. The parametrizations utilize four convolution layers followed by four fully connected layers. The calibrated residual approximation finally has an additional diagonal layer to approximate the noise level.
Training is done for 30 epochs and takes roughly $120$ seconds per epoch.

The computational times using MATLAB's PDE solver and the NLL approximation are reported in Table~\ref{tab: computational times}. 
The CPU-based runtimes correspond to the same computational hardware, whereas the GPU-based NLL approximation is performed on a GPU cluster. For the approximations, we display computational runtimes for parallel evaluations, which could be used to generate many chains in parallel, which is however outside the scope of our numerical experiment. On the same hardware, the NLL approximations are $\approx 200$ times faster compared with the PDE model.
    \begin{table}[h]
        \centering
        \caption{Computation time in seconds for the NLL and its surrogates. Whenever possible, the 100 and 1000 evaluations are made in parallel. For the free-form, residual and calibrated residual parametrizations, computational times are essentially identical. The NLL approximations are at least $\approx 200$ times faster compared to the finite elements PDE model on the same device and up to $20,000$ time faster under parallelization on a GPU device.}
        \begin{tabular}{c|c c c c}
            & PDE solve &NLL approx. &  NLL approx.  \\
            & (CPU)  & (CPU) & (GPU) \\
            \hline
            single evaluation & 0.7008 & 0.0042 &0.0016 \\
            100   evaluations    & 70.0083 &  0.3972& 0.0036\\
            1000  evaluations  &  700.8333& 3.6880&0.0331 
        \end{tabular}
        \label{tab: computational times}
    \end{table}
\paragraph{MCMC sampling}
We fix a previously unseen pair $(x, y) \sim \lambda$ and use a preconditioned Crank-Nicolson (pCN) sampler to sample from the posterior conditional on the measurement $y$. In the pCN algorithm, we use a truncated KL-expansion to sample from the Gaussian random field. For all approaches, we generate $10,000$ samples sequentially, discard the first $5,000$ as burn-in and finally thin the chain by factor five. 

The computational effort for posterior sampling using the NLL approximation varies due to the underlying posterior variance and acceptance rate.  The computational times for the NLL approximation is $< 50$ minutes on CPU and $<10$ minutes on a GPU cluster.
In contrast, utilizing the discretized PDE model on a CPU device, posterior sampling takes $265$ hours.

For the NLL approximation, we require the offline computations of data generation and training, which together take less than 6 hours. In this time it would be possible to generate approximately 440 samples with the PDE model.

Posterior means, biases, and variances are summarized in Figure~\ref{fig: MCMCfull semiconductor}. We see that the biases and variances of the posterior with the true PDE model and the free-form and residual approximations are very similar. The spatial $L^1$ norm of the biases are given by $0.1807$ (exact), $0.1839$ (free-form) and $0.1954$ (residual), confirming that the bias of the true PDE model is extremely close to the free-form and residual approximations. In contrast, the calibrated residual approximation fails to generate accurate posterior variance and means.

This experiment demonstrates that neural likelihood approximation can lead to enormous computational acceleration in PDE-based, non-linear inverse problems. The posterior summary statistics, such as bias and variance, are highly accurate for the free-form and residual approximations. In settings where the observational noise level is unknown, the free-form approximation can recover the posterior with high fidelity. The approximation methods require only joint data for training and no evaluations of the true forward map or likelihood and are thus highly versatile.
\begin{figure}[ht!]
    \centering
    \begin{subfigure}[b]{0.32\textwidth}
        \centering
        True $C$
        \includegraphics[width=\textwidth]{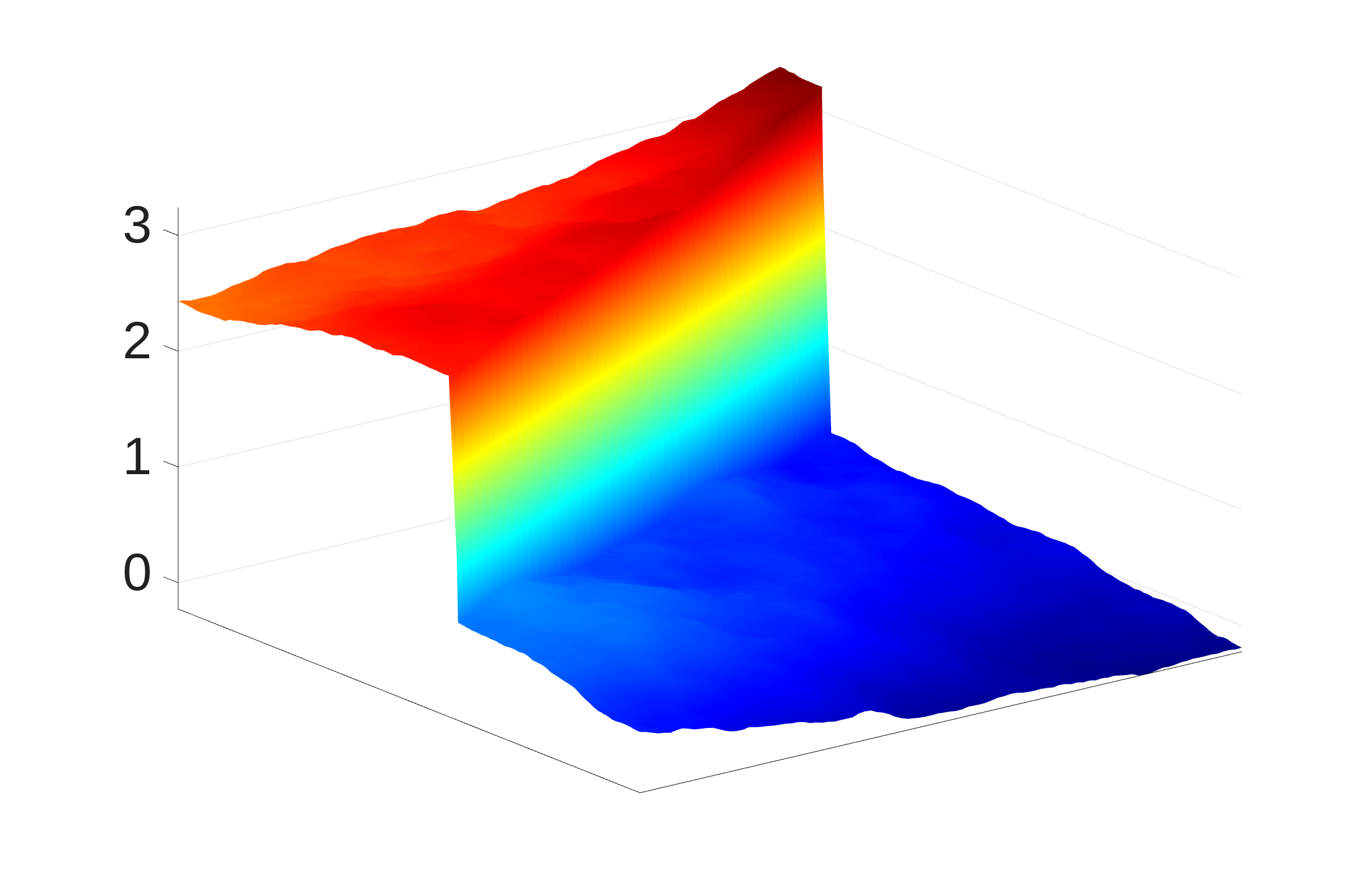}
    \end{subfigure} 
    \\
        \begin{subfigure}[b]{0.3\textwidth}
        \centering
        \includegraphics[width=\textwidth]{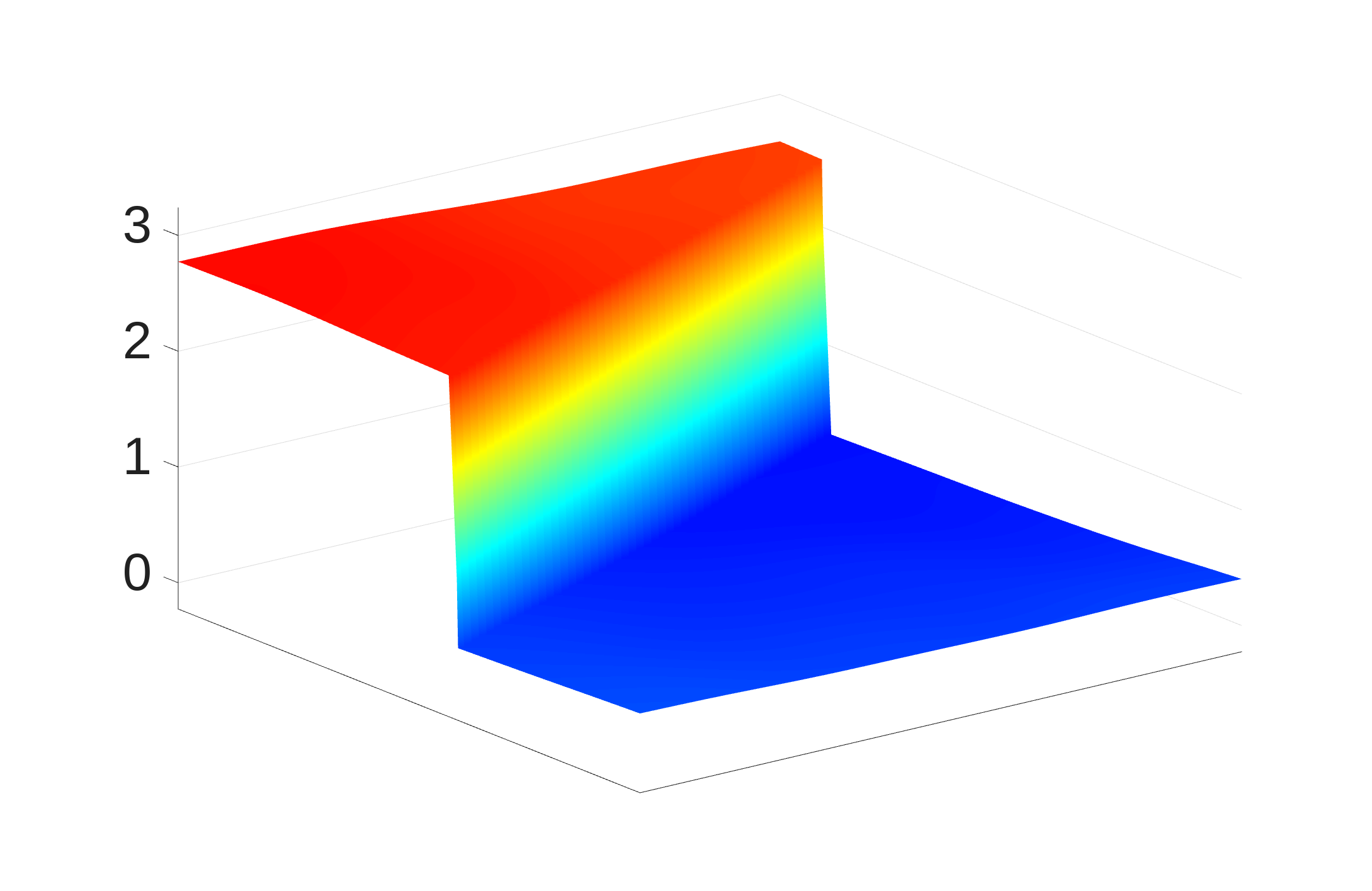}
    \end{subfigure}
            \begin{subfigure}[b]{0.3\textwidth}
        \centering
        Exact, PDE-based NLL\\
        \includegraphics[width=\textwidth]{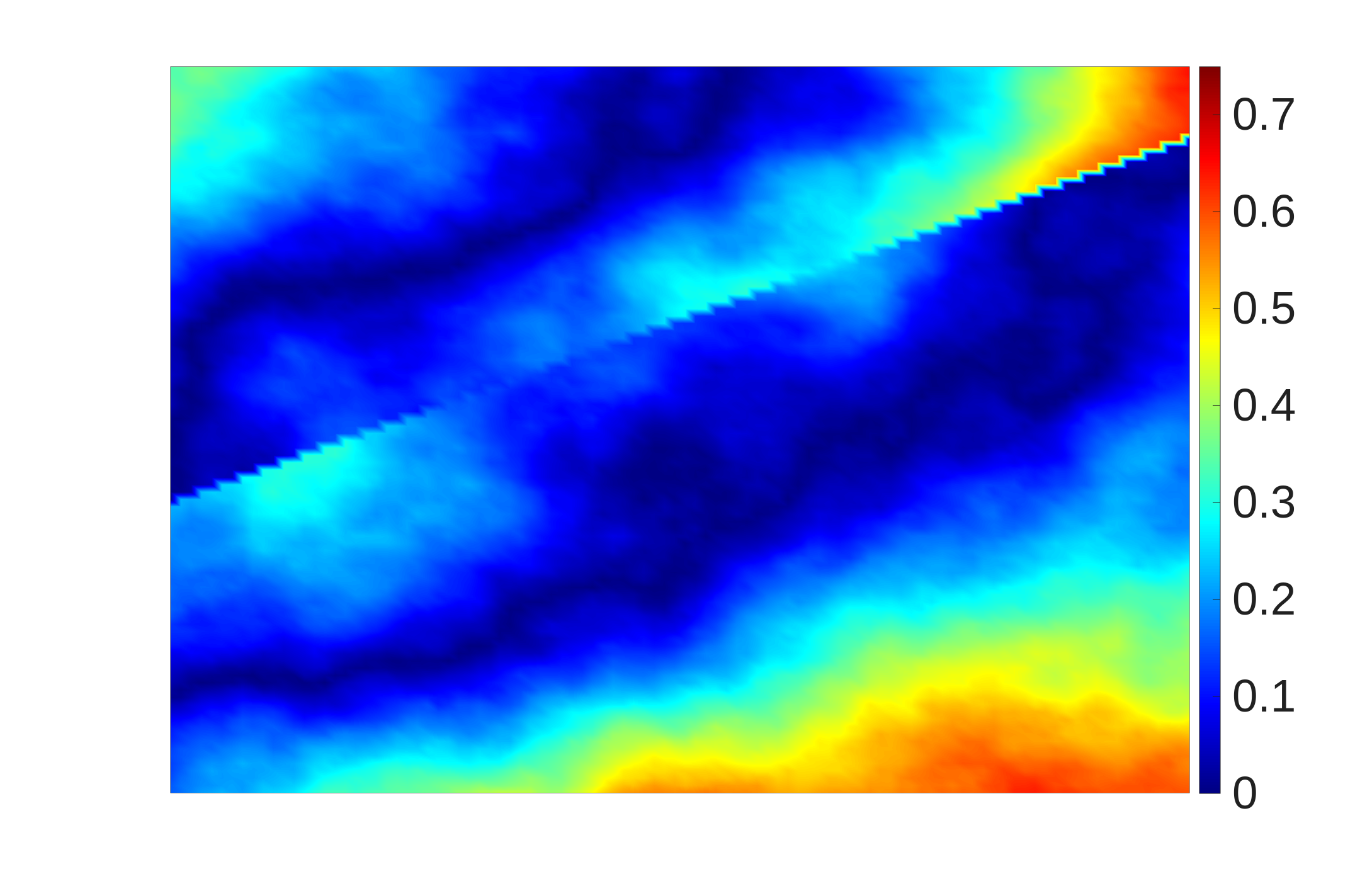}
    \end{subfigure}
    \begin{subfigure}[b]{0.3\textwidth}
        \centering
        \includegraphics[width=\textwidth]{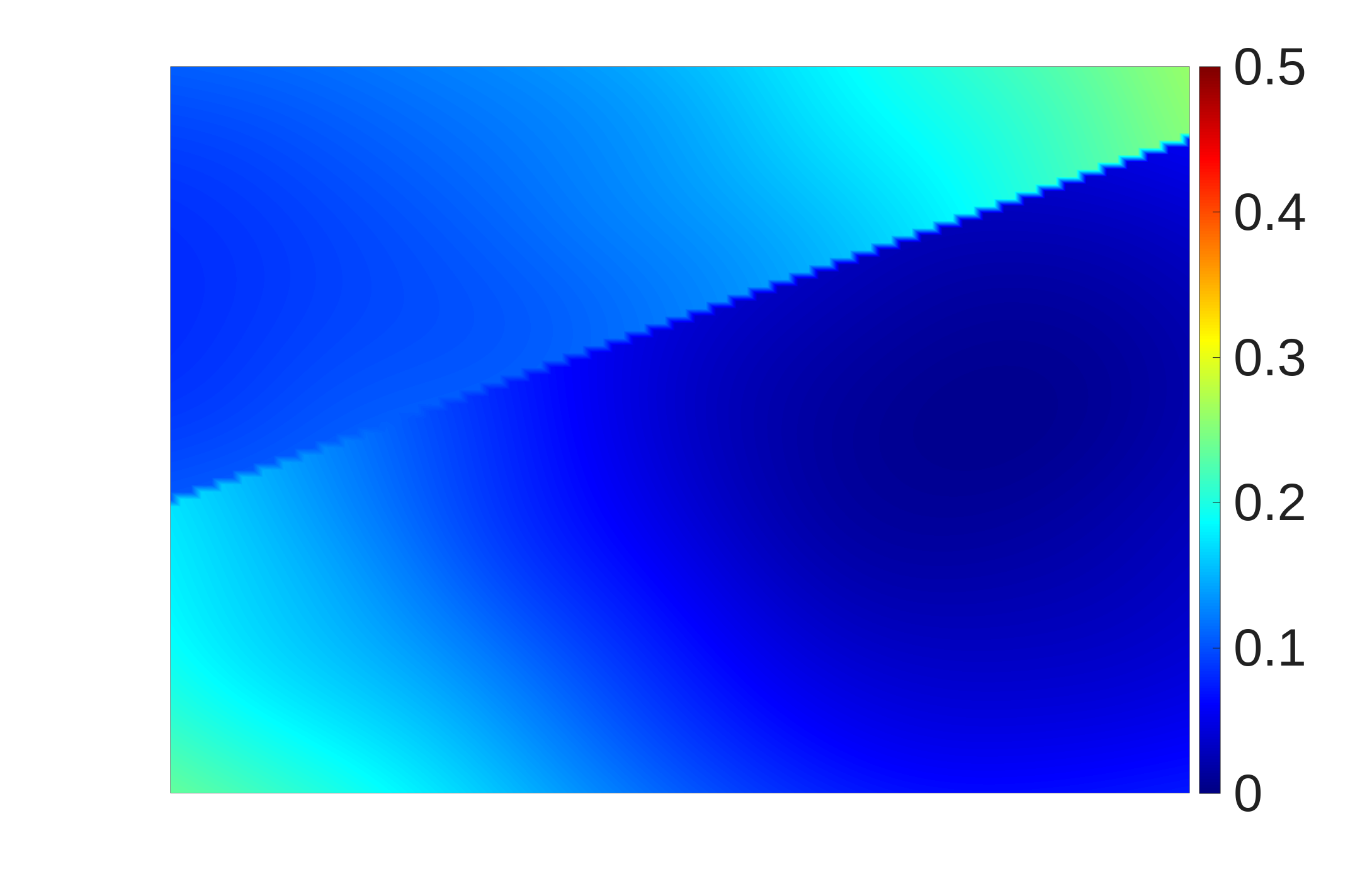}
    \end{subfigure} 

           \begin{subfigure}[b]{0.3\textwidth}
        \centering
        \includegraphics[width=\textwidth]{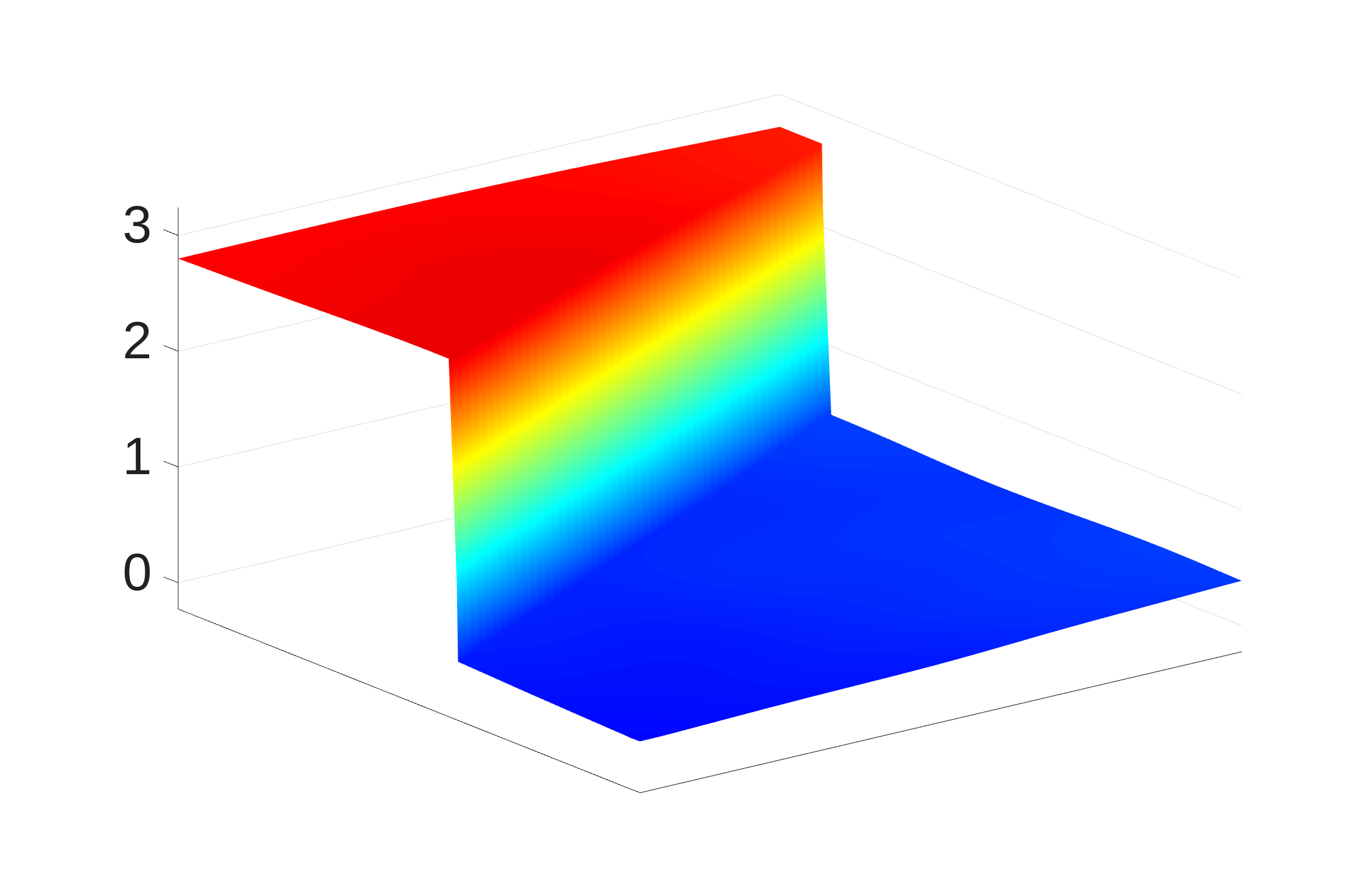}
    \end{subfigure}
            \begin{subfigure}[b]{0.3\textwidth}
        \centering
        Free-form approximation\\
        \includegraphics[width=\textwidth]{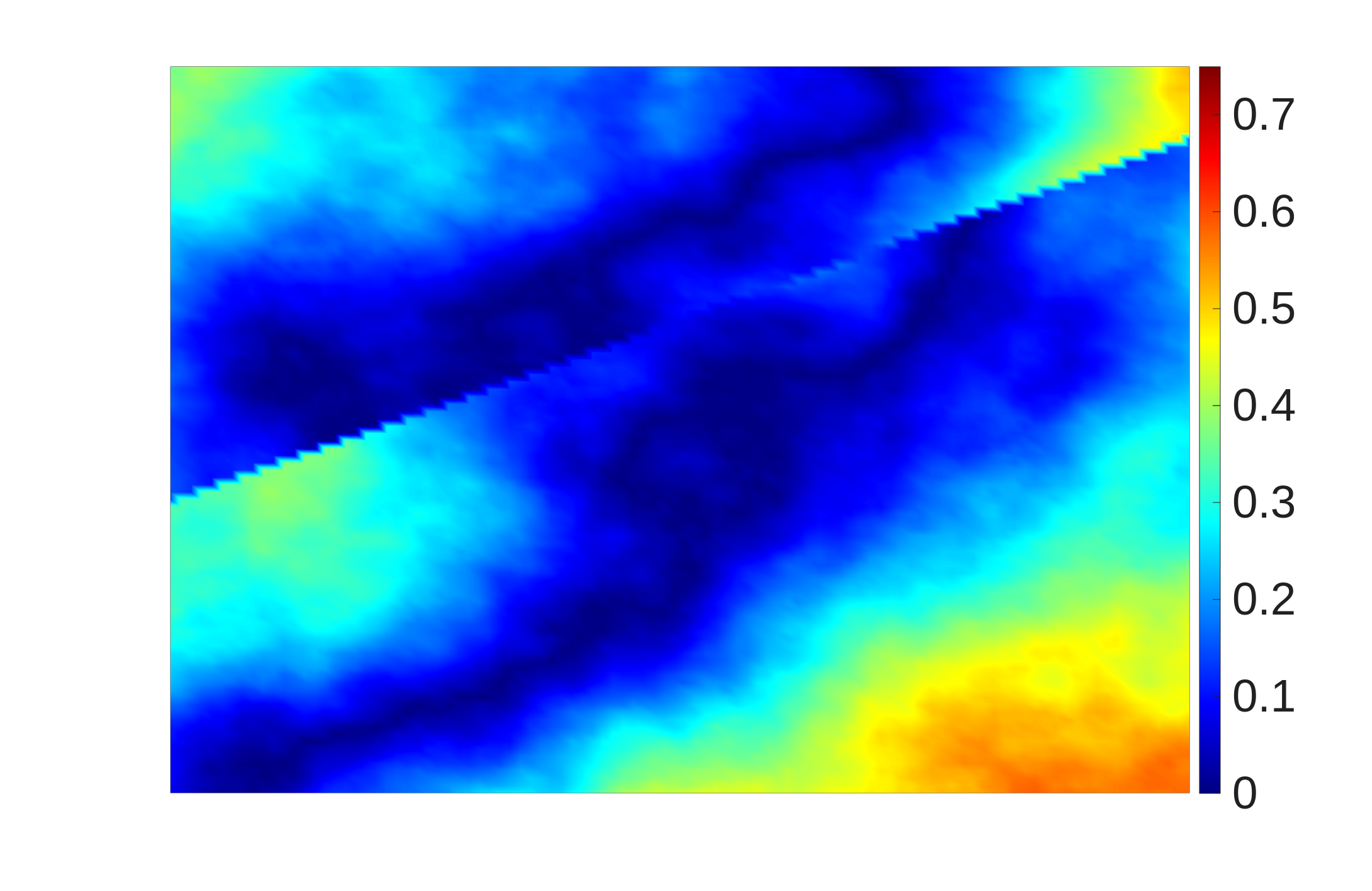}
    \end{subfigure}
    \begin{subfigure}[b]{0.3\textwidth}
        \centering
        \includegraphics[width=\textwidth]{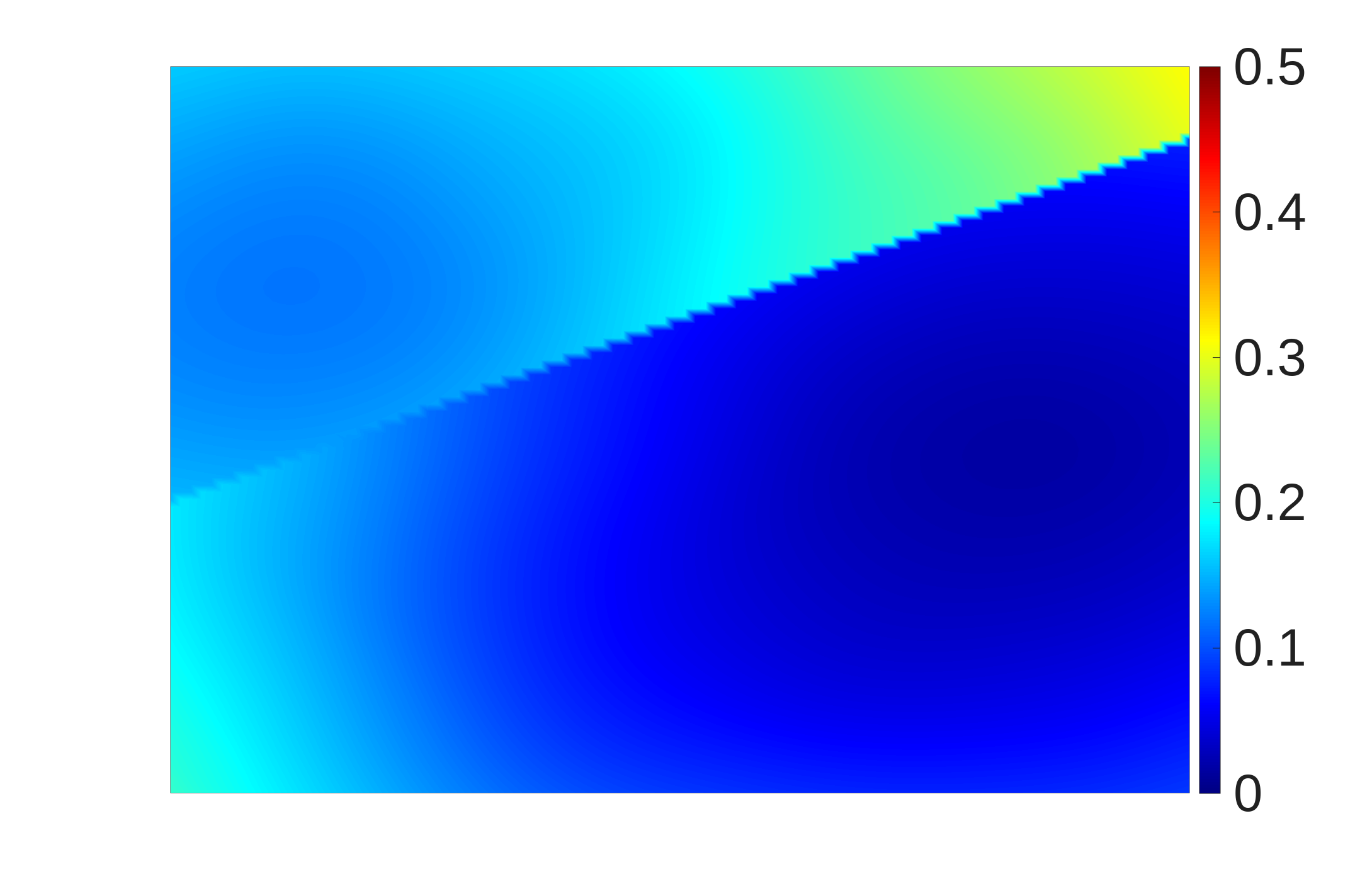}
    \end{subfigure}

               \begin{subfigure}[b]{0.3\textwidth}
        \centering
        \includegraphics[width=\textwidth]{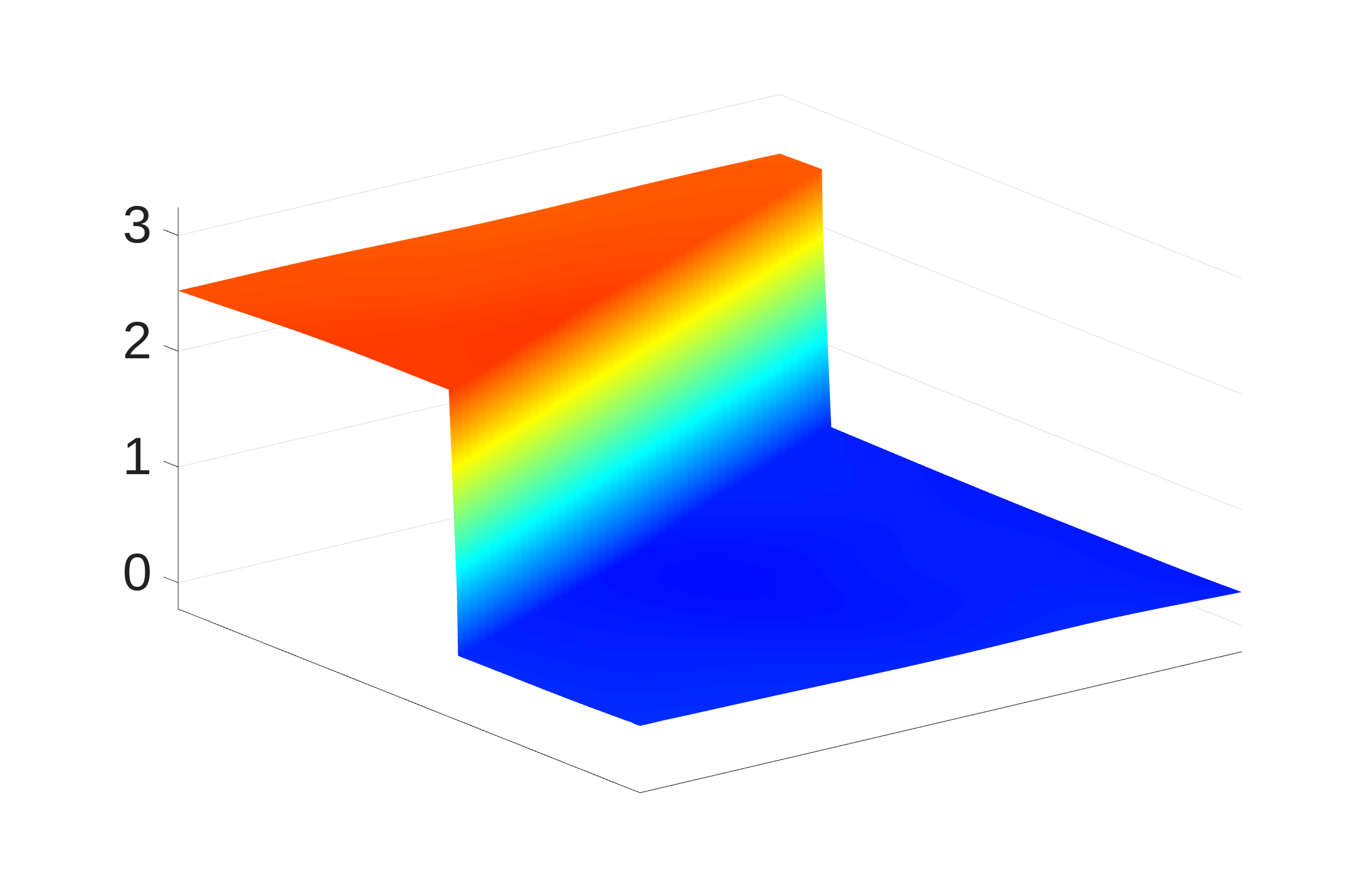}
    \end{subfigure}
            \begin{subfigure}[b]{0.3\textwidth}
        \centering
        Residual approximation\\
        \includegraphics[width=\textwidth]{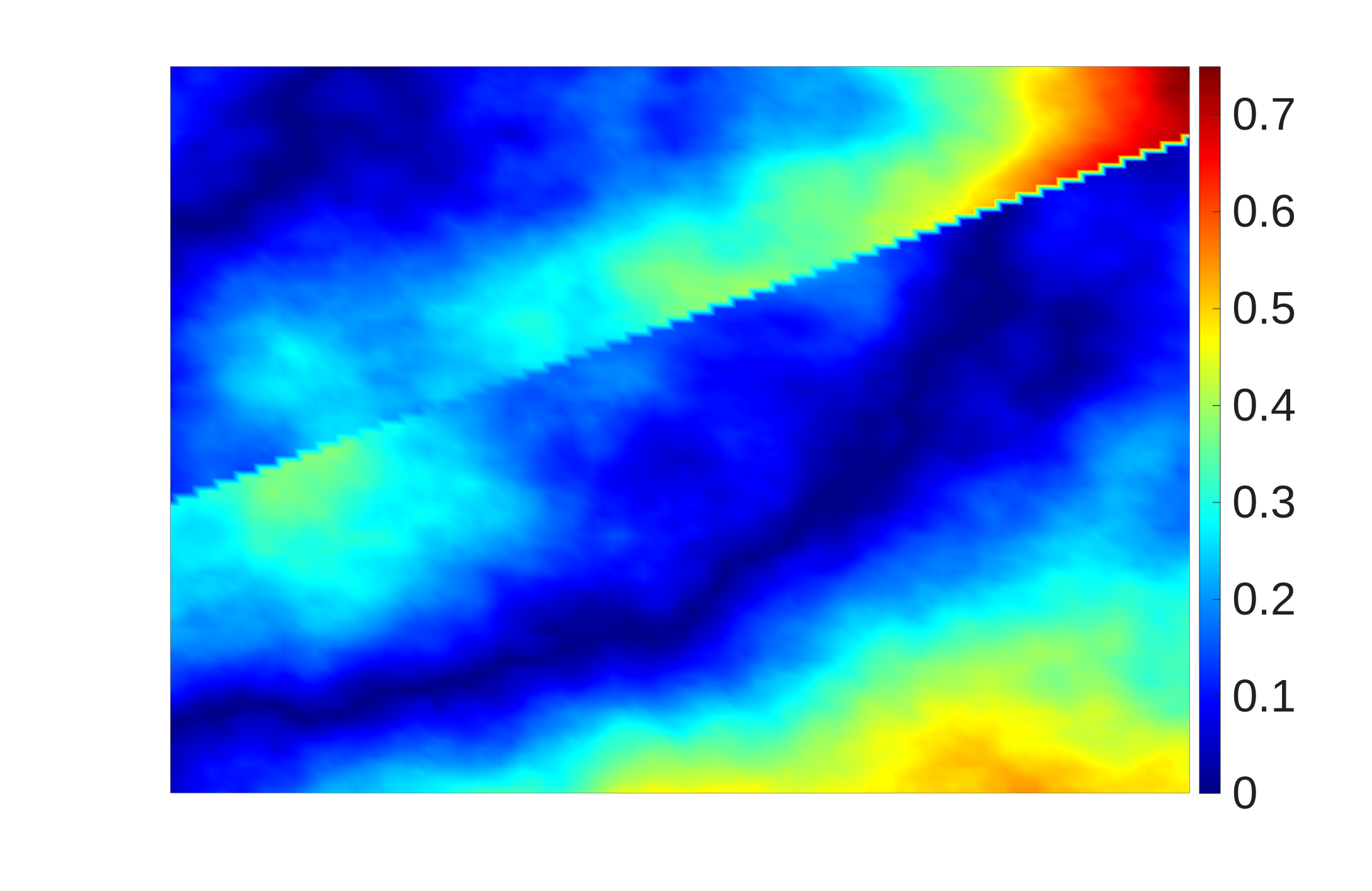}
    \end{subfigure}
    \begin{subfigure}[b]{0.3\textwidth}
        \centering
        \includegraphics[width=\textwidth]{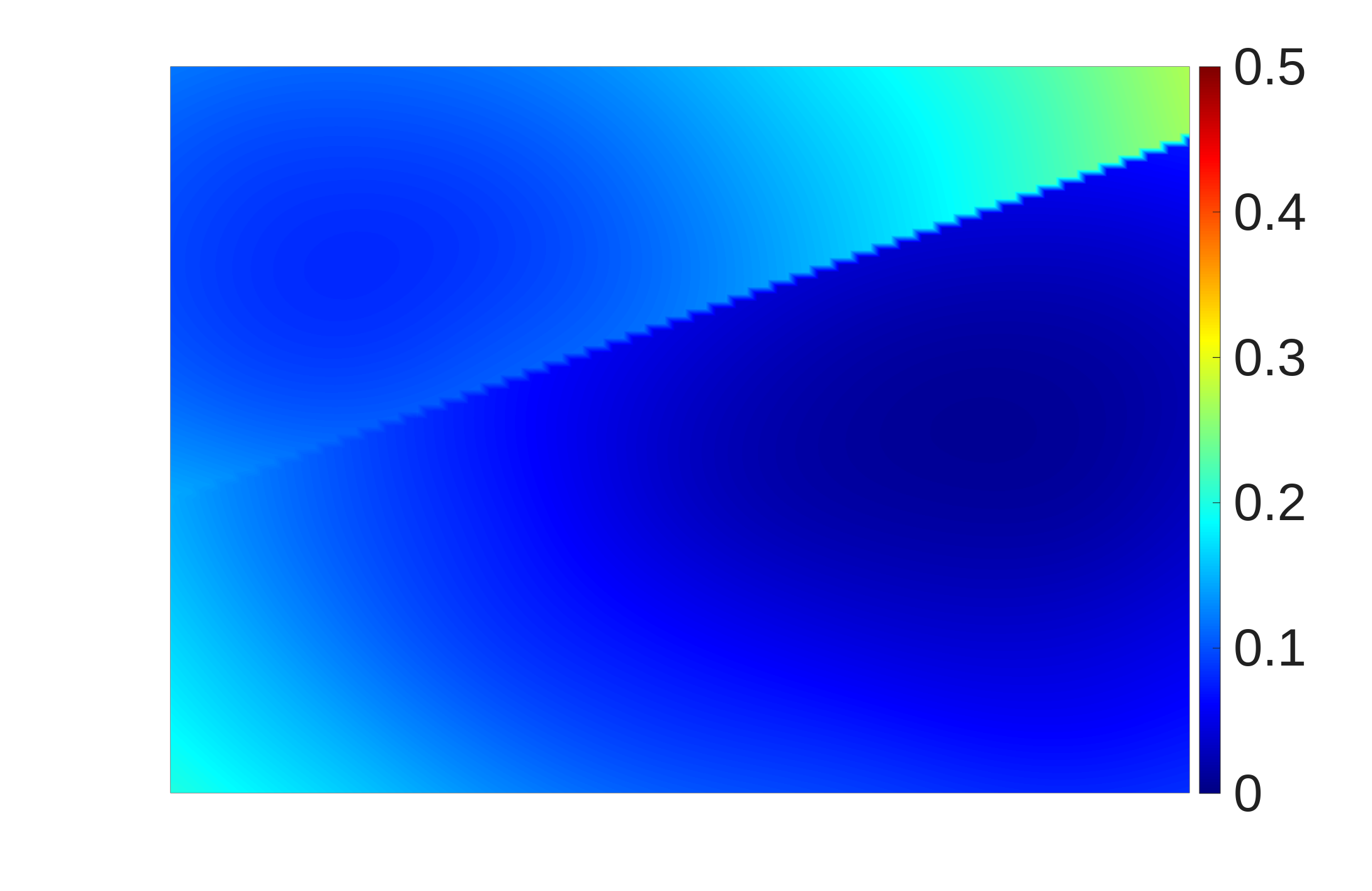}
    \end{subfigure}

               \begin{subfigure}[b]{0.3\textwidth}
        \centering
        \includegraphics[width=\textwidth]{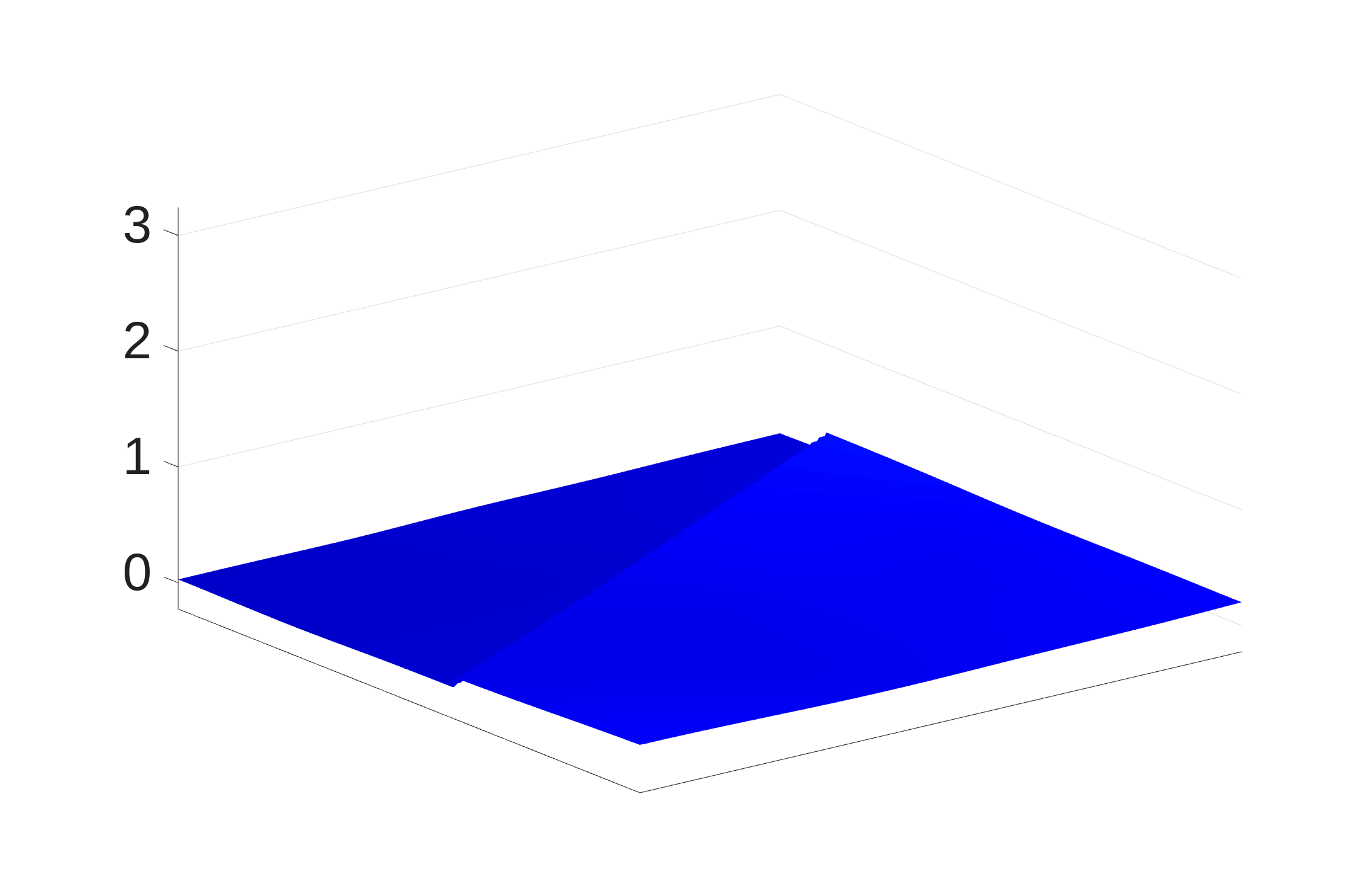}
    \end{subfigure}
            \begin{subfigure}[b]{0.3\textwidth}
        \centering
        Calibrated residual approximation\\
        \includegraphics[width=\textwidth]{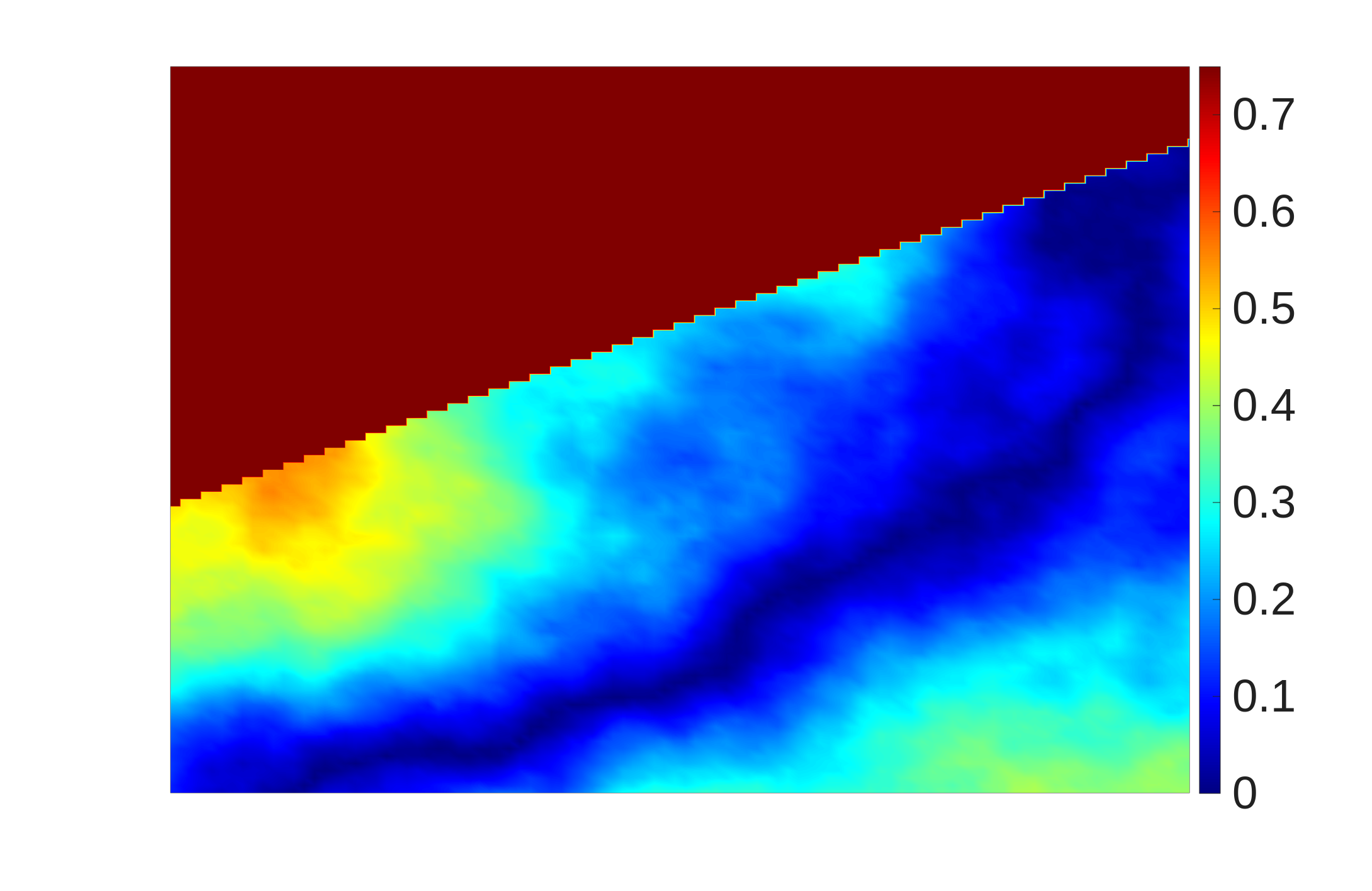}
    \end{subfigure}
    \begin{subfigure}[b]{0.3\textwidth}
        \centering
        \includegraphics[width=\textwidth]{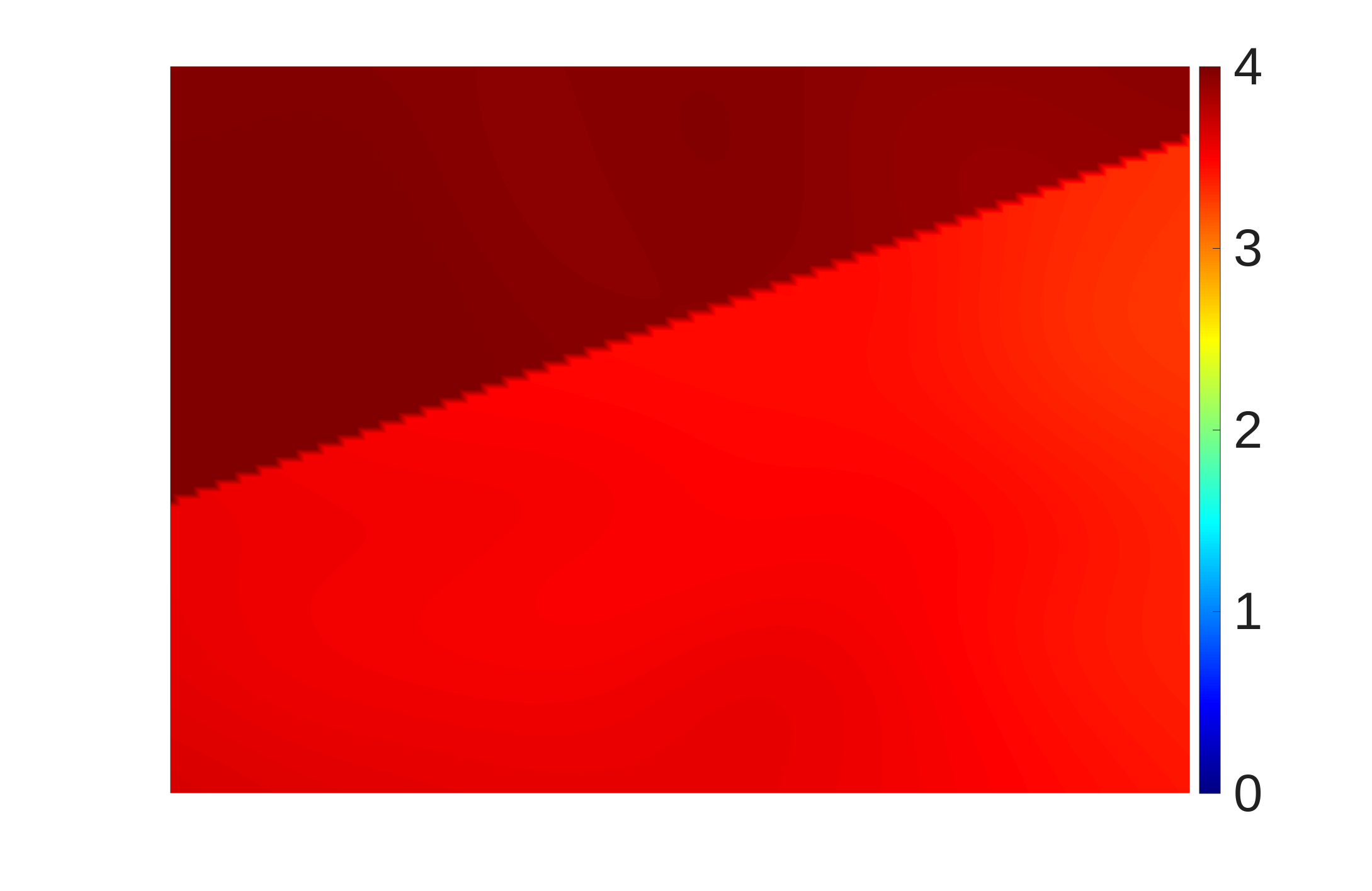}
    \end{subfigure}
    \hfill
    \caption{Summary of MCMC samples. Top row: True quantity of interest followed by posterior mean (left), bias (middle) and posterior variance (right) of the posterior samples with the true PDE-based NLL and its approximations. The $L^1$ norm of the biases are given by $0.1807$ (exact), $0.1839$ (free-form), $0.1954$ (residual), $1.072$ (calibrated residual).} 
    \label{fig: MCMCfull semiconductor}
\end{figure}


\section{Conclusions}\label{s:Conclusions}
This work developed a theoretical foundation for neural likelihood approximation by overcoming limitations imposed by restrictive parametric model classes. We introduced a normalization scheme in which an unconstrained function is mapped to a valid probability density by division by its expectation, thereby embedding normalization directly into a KL-based learning objective. This eliminates the need to enforce normalization through architectural constraints while preserving full flexibility of the approximative class. We showed that the resulting objective is convex in the un-normalized NLL. Moreover, we established consistency of the estimator under empirical approximation of the normalization integral, ensuring convergence to the true target in the large-sample regime.

We demonstrated the practical effectiveness of the proposed approach on a deblurring task and a non-linear PDE-based imaging problem, where only joint $(x, y)$ training data is available but the forward model is assumed unknown. We studied different parametrizations of the NLL and found that a free-form neural-network based approximation can accurately approximate the NLL, even without knowledge of the observational noise level, with significant computational acceleration. A residual-based parametrization struggles when the observational noise level is not known but performs well with known observational noise.

Quantification of the convergence through convergence rates remains an open problem. A further current limitation is that the consistency result is established under finite-dimensional spaces supporting the unknown and measurement due to the scope of existing empirical process theory used in the analysis. Extending these results to fully infinite-dimensional function classes remains an important direction for future work.

\FloatBarrier

\appendix
\section{Auxiliary results}

To prove Lemma~\ref{lem: subgaussian joint}, we show the following assertion, which is a counterpart to \cite[Lemma 3.11]{helin2025bayesianoptimalexperimentaldesign}
\begin{lemma}\label{lem: Z_fy bounds}
    There exists universal constants $C^-, C^+$ such that for all $f \in \F$
    \begin{eqnarray*}
        C^-  e^{-C_2^+ \norm{y}^2}   \le Z_f(y) \le C^+ e^{- \delta(C_1) \norm{y}^2} .
    \end{eqnarray*}
\end{lemma}

\begin{proof}
    Using \eqref{eq: bounds on f} with $\varepsilon = C_1$
    \begin{eqnarray*}
        Z_f(y) = \E^\mu \left[ e^{-f(x; y)} \right] \le e^{C_2^-}  e^{- \delta(C_1) \norm{y}^2} \, \E^\mu \left[ e^{C_1 \norm{x}^2} \right].
    \end{eqnarray*}
    Using the lower bound of $\eqref{eq: bounds on f}$:
    \begin{eqnarray*}
        Z_f(y) = \E^{\mu } \left[e^{-f(x; y)}\right] \ge e^{-C_2^+ (1+\norm{y}^2)} \E^\mu \left[ e^{- C_2^+ \norm{x}^2} \right] .
    \end{eqnarray*}
\end{proof}

\begin{proof}[Proof of Lemma~\ref{lem: subgaussian joint}] We choose $\kappa:=  \min\left(C_1, \delta(C_1) \right) / 3$. It holds
    \begin{eqnarray*}
        \E^{\lambda_f} \left[ e^{\kappa (\norm{x}^2 + \norm{y}^2)} \right] & \le & \left( \E^{\lambda_f} \left[e^{2 \kappa \norm{x}^2} \right]  \right)^{1/2}  \left( \E^{\lambda_f} \left[e^{2 \kappa \norm{y}^2} \right]  \right)^{1/2} .
    \end{eqnarray*}
    We will now bound the terms on the right-hand side in the previous equation.
    Using Lemma~\ref{lem: Z_fy bounds}:
    \begin{eqnarray*}
        \E^{\lambda_f} \left[e^{2 \kappa \norm{y}^2} \right] &=& \E^{\pi_f} \left[e^{2 \kappa \norm{x}^2} \right] = \frac{\int_{\R^m} e^{2 \kappa \norm{y}^2}  Z_f(y) dy }{\int_{\R^m} Z_f(y) dy} \\
         & \le &\frac{ \int_{\R^m} C^+ e^{2 \kappa \norm{y}^2} e^{-\delta(C_1) \norm{y}^2} dy}{\int_{\R^m} Z_f(y) dy} .
    \end{eqnarray*}
   By Lemma~\ref{lem: Z_fy bounds}
    \begin{eqnarray}\label{eq: expected Zf lower bound}
        \int_{\R^m}Z_f(y) dy\ge C^- \int_{\R^m}    e^{-C_2^+ \norm{y}^2} dy > 0,
    \end{eqnarray}
    we have
    \begin{eqnarray*}
        \E^{\lambda_f} \left[e^{2 \kappa \norm{y}^2} \right] &\le &\frac{\int_{\R^m} C^+ e^{2 \kappa \norm{y}^2} e^{-\delta(C_1) \norm{y}^2} dy}{\int_{\R^m}  C^-  e^{-C_2^+ \norm{y}^2} dy}.
    \end{eqnarray*}
    By construction, $2\kappa < \delta(C_1)$ so the integral is finite. Regarding the second term, it holds by \eqref{eq: joint}, \eqref{eq: bounds on f} and \eqref{eq: expected Zf lower bound}:
    \begin{eqnarray*}
        \E^{\lambda_f} \left[e^{2 \kappa \norm{x}^2} \right]  =    \E^{\pi_f \otimes \mu} \left[e^{2 \kappa \norm{x}^2} \frac{e^{-f(x, y)}}{\int_{\R^m} Z_f(y) dy} \right] \le \frac{e^{C_2^-}}{C^- \int_{\R^m }e^{-C_2^+ \norm{y}^2}} \E^{\mu} \left[e^{2.5 \kappa \norm{x}^2} \right] 
    \end{eqnarray*}
    This term is indeed finite for $\kappa \le  C_1/ 3$ by \eqref{eq: sub gaussian prior}.
\end{proof}

To prove Lemma \ref{lem: lambda bounded by Lp}, we need the following results.
\begin{lemma}\label{lem: rn^q for marginal}
    Let $r < 1 + \frac{\kappa}{C_2^+}$. Then there exists a $C(r)>0$ such that for all $f, g\in \F$
    \begin{eqnarray*}
        \norm{\frac{\pi_g(y)}{\pi_f(y)}}_{L^{r}(\pi_f)} < C(r).
    \end{eqnarray*}
\end{lemma}
\begin{proof}
We use the definition of $\pi_h$ for $h \in \{f, g\}$ in \eqref{eq: marginal of y corresponding to L} to get
    \begin{align*}
        \int_{\R^m} \pi_g(y)^{r} \pi_f(y)^{1-r} \, dy \le \left(\int_{\R^m} Z_g(y) dy\right)^{-r} \left(\int_{\R^m} Z_f(y) dy\right)^{r-1} \int_{\R^m} Z_g(y)^{r} Z_f(y)^{1-r} dy
    \end{align*}
   We deal with the integral first. By Lemma~\ref{lem: Z_fy bounds},
    \begin{eqnarray*}
        \int_{\R^m} Z_g(y)^{r} Z_f(y)^{1-r} 
        & \le & (C^+)^r (C^-)^{1-r} \int_{\R^m} e^{(-r \delta(C_1) + (r-1) C_2^+) \norm{y}^2} \\
        & \le & (C^+)^r (C^-)^{1-r} \int_{\R^m} e^{(r-1) C_2^+ \norm{y}^2}.
    \end{eqnarray*}
    Since $(r-1) C_2^+ < \kappa$, this term is indeed uniformly bounded by Lemma~\ref{lem: subgaussian joint}.
    The second factor is now dealt with by utilizing Lemma~\ref{lem: Z_fy bounds}:
    \begin{eqnarray*}
        \left(\int_{\R^m}  Z_g(y) dy\right)^{-r} \left(\int_{\R^m}  Z_f(y) dy\right)^{r-1} \le  \left(C^-\int_{\R^m}  e^{-C_2^+ \norm{y}^2} dy\right)^{-r} \left(C^+\int_{\R^m}  e^{-\delta(C_1) \norm{y}^2} dy\right)^{r-1}  < \infty.
    \end{eqnarray*}
\end{proof}
\begin{lemma}\label{lem: rn^q}
    For $q_0:= 1 +\frac{\kappa}{2 C_2^+ r'}$ and any $1 < q < q_0$ such that for given $g \in \F$ and all $f \in \F$ we have the there exists $C(q)$ such that
    \begin{eqnarray*}
       \E^{\lambda_g} \left( \frac{ d \mu^y_f}{d \mu_g^y}\right)^q < C(q)
    \end{eqnarray*}
\end{lemma}

\begin{proof}
In a first step, we introduce the Hölder exponents $r:= 1 + \frac{\kappa}{2 C_2^+}$, $r' = 1 + \frac{2 C_2^+}{\kappa}$ and write
    \begin{eqnarray*}
         \E^{\lambda_g}  \left[ \left( \frac{ d \mu^y_f}{d \mu_g^y}\right)^q \right] &= &\E^{\pi_g} \E^{\mu_g^y}  \left[ \left( \frac{ d \mu^y_f}{d \mu_g^y}\right)^q \right]=  \E^{\pi_g} \E^{\mu^y_f}  \left[ \left( \frac{ d \mu^y_f}{d \mu_g^y}\right)^{q-1} \right] =  \E^{\pi_f} \left[ \frac{ d \pi_g}{d \pi_f} \, \E^{\mu^y_f}  \left[ \left( \frac{ d \mu^y_f}{d \mu_g^y}\right)^{q-1} \right] \right] \\
         &\le &  \left[ \E^{\lambda_f} \left( \frac{ d \mu^y_f}{d \mu_g^y}\right)^{r'(q-1)} \right]^{1/{r'}} 
        \left[\E^{\pi_f}  \left( \frac{\pi_g(y)}{\pi_f(y)}\right)^{r} \right]^{1/{r}}.
    \end{eqnarray*}
    The second term involving the integral over $\pi_f$ is uniformly bounded by Lemma~\ref{lem: rn^q for marginal}.
    To deal with the other term, we realize that  $f(x,y) \ge 0$ and hence $Z_g(y) \le 1$ for all $(x, y) \in H \times \R^m$ such that
    \begin{eqnarray*}
         \left| \frac{ d \mu^y_f}{d \mu_g^y}\right|^{r'(q-1)} &=& \frac{e^{r' (q-1) (g(x, y)-f(x, y))}}{Z_f(y)^{r'(q-1)} Z_g(y)^{r'(1-q)}} 
         \le  \frac{e^{r' (q-1) g(x, y)}}{Z_f(y)^{r'(q-1)} }  \\
         &\le&  \tfrac{1}{C^-} e^{r' (q-1) C_2^+ (\norm{x}^2 + \norm{y}^2)}
         e^{r'(q-1)C_2^+ \norm{y}^2}.
    \end{eqnarray*}
    Equation \eqref{eq: bounds on f} together with Lemma~\ref{lem: Z_fy bounds} was used.
    Observe that $2 r'(q-1)C_2^+ \le \kappa$. Therefore
    \begin{eqnarray*}
        \E^{\lambda_f} \left[ \left| \frac{ d \mu^y_f}{d \mu_g^y}\right|^{r'(q-1)} \right]  \le \tfrac{1}{C^-} \, \E^{\lambda_f} \left[e^{\kappa (\norm{x^2} + \norm{y}^2)} \right],
        \end{eqnarray*}
    which is indeed uniformly bounded by Lemma~\ref{lem: subgaussian joint}.
\end{proof}

\begin{proof}[Proof of Lemma~\ref{lem: lambda bounded by Lp}]
    Fix in the first step $y \in \R^m$. Define $\Lambda^y(f) := \log \E^{\mu} \left[e^{-f(x, y)}\right]$.  For $t \in [0, 1]$ we define $k_t = g+t(f-g) \in \F$, since $\F$ is convex. By the fundamental theorem of calculus,
    \begin{eqnarray}
    \begin{aligned}\label{eq: pf of lp lemma, eq for fund of calc}
         \Lambda^y(f) - \Lambda^y(g) &= \int_0^1 \frac{d}{dt} \Lambda^y(k_t) dt 
          =  \int_0^1 \frac{\frac{d}{dt} \E^{\mu} \left[ e^{-k_t(x, y) } \right]}{\E^{\mu} \left[ e^{-k_t(x, y)} \right]  \, dt}.  
         \end{aligned}
    \end{eqnarray}
    By Leibniz integral rule 
    \begin{eqnarray}\label{eq: pf of lp lemma, eq deriv vs int}
        \frac{d}{dt} \E^{\mu} \left[ e^{-k_t(x, y)} \right] = \E^{\mu} \left[  \left(f(x, y) - g(x, y)\right)  e^{-k_t(x, y) } \right].
    \end{eqnarray}
    The domination function $G \in L^1(\mu)$ in Leibniz integral rule is given by
    \begin{eqnarray*}
        \left|f(x, y) - g(x, y)\right|  e^{-k_t(x, y) } \le 2C_2^+ (\norm{x}^2 + \norm{y}^2) := G(x).
    \end{eqnarray*} 
    Putting together \eqref{eq: pf of lp lemma, eq for fund of calc} and \eqref{eq: pf of lp lemma, eq deriv vs int}, we get
    \begin{eqnarray*}
        \Lambda^y(f) - \Lambda^y(g) & = & \int_0^1 \E^{\mu^y_{k_t}} \left[f(x, y) - g(x, y)\right]  \, dt.
    \end{eqnarray*}
    Therefore by Fubinis' Theorem and for any $1/p + 1/q = 1$
    \begin{eqnarray*}
        \E^{\pi_h} |\Lambda^y(f) - \Lambda^y(g)| &=& \int_0^1 \E^{\lambda_h} \left[ \frac{d \mu^y_{k_t}}{d \mu_h^y}  \left( g(x, y) - f(x, y\right) \right] dt \\
        & \le &  \left[ \E^{\lambda_h}  \left| g(x, y) - f(x, y)\right|^p  \right]^{1/p}  \int_0^1 \left( \E^{\lambda_h}\left| \frac{d \mu^y_{k_t}}{d \mu_h^y} \right|^q  \right)^{1/q} dt.
    \end{eqnarray*}
    Notice that $k_t \in \F$ and we can apply Lemma~\ref{lem: rn^q} and fix now $q < q_0$ and the corresponding $p$ that satisfies $1/p + 1/q = 1$.
\end{proof}
\section{Proofs for data-driven estimates}
\subsubsection*{Proof of Theorem \ref{lem: conv in p}}
\begin{proof}
   We introduce the decomposition
    \begin{align}\label{eq: decomp prob in p}
        \Lambda_{N}(f) - \Lambda(f) =   \frac{1}{N} \sum_{i=1}^N \left(\log Z_f^{M(N)}(y_i) - \log Z_f(y_i)\right)+ \left(\frac{1}{N} \sum_{i=1}^{N} \log Z_f(y_i) - \Lambda(f)\right)
    \end{align}
    and show convergence to zero in probability for the first term (I) and the second term (II) separately.
    
    \textbf{(I)} We define for fixed $i = 1, \ldots, N$ 
    \begin{eqnarray*}
        \xi_i^{M(N)}:= \log Z_f^{M(N)}(y_i) - \log Z_f(y_i).
    \end{eqnarray*}
     We show now that $\xi_i^{M(N)}$ is uniformly integrable.
    \begin{eqnarray*} 
            \sup_{N \in \Natural} \E^{\mu\otimes \pi} (\xi_1^{M(N)})^2 &\le& 2\left(\log \E^{\mu\otimes \pi} e^{-f(x, y_i)} \mu(dx)\right)^2 +2 \sup_{N \in \Natural} \E^{\mu\otimes \pi} \left(\log \frac{1}{M(N)} \sum_{j=1}^{M(N)} e^{-f(x_{i, j}, y_i)}\right)^2.
    \end{eqnarray*}
    For the first term we use \eqref{eq: sub gaussian prior} and \eqref{eq: bounds on f} to find
    \begin{eqnarray*}
        \left(\log \E^{\mu\otimes \pi} e^{-f(x, y_i)} \mu(dx)\right)^2 \le \left(C_2 ^- + \log \E^{\mu} 2e^{C_1 \norm{x}^2} \mu(dx) \right)^2 < \infty.
    \end{eqnarray*}
    And for the second term, we find upper and lower bounds for the logarithm-term by using \eqref{eq: sub gaussian prior} and Jensen's inequality:
    \begin{eqnarray*}
         \log \frac{1}{M(N)} \sum_{j=1}^{M(N)} e^{-f(x_{i, j}, y_i)} \begin{cases}
             \le C_2^- + \log \frac{1}{M(N)} \sum_{j=1}^{M(N)} e^{C_1/2 \norm{x_i, j}^2} \\
              \ge -C_2^+ - C_2^+ \norm{y_i}^2 - \frac{C_2^+}{M(N)} \sum_{j=1}^{M(N)} \norm{x_{i, j}}^2.
         \end{cases}  
    \end{eqnarray*}
    By using $\log z \le z$ for $z \ge 1$ we find
    \begin{eqnarray*}
        \sup_{N \in \Natural} \E^{\mu} \left(C_2^- + \log \frac{1}{M(N)} \sum_{j=1}^{M(N)} e^{C_1/2 \norm{x_{i, j}}^2}\right)^2 & \le & 2 (C_2^-)^2 + 2 \E^{\mu} \left(e^{C_1/2 \norm{x}^2}\right)^2 \\
        &\le& 2 (C_2^-)^2 + 2 \E^{\mu} e^{C_1 \norm{x}^2} <  \infty.
    \end{eqnarray*}
    Through Jensen's inequality we find
    \begin{eqnarray*}
        & &\sup_{N \in \Natural} \E^{\mu\otimes \pi} \left( (1+ C_2^+) \norm{y_i}^2 + \frac{C_2^+}{M(N)} \sum_{j=1}^{M(N)} \norm{x_{i, j}}^2 \right)^2
        \le 2 (1+C_2^+)^2 \E^{\pi}\norm{y}^4 + 2 (C_2^+)^2\sup_{N \in \Natural} \E^{\mu}  \norm{x}^4.
    \end{eqnarray*}
    In summary, we showed that $ \sup_{N \in \Natural} \E^{\mu} (\xi_1^{M(N)})^2 < \infty$ and hence $\xi_i^{M(N)}$ is uniformly integrable. Since $\xi_1^{M(N)} \rightarrow 0$ converges almost surely by the law of large numbers (notice that $x_{i, j}$ are all iid) and the continuity of the logarithm,  and is uniformly integrable, it holds
    \begin{eqnarray}\label{eq: toshow in conv in prob}
         \E^{\mu \otimes \pi} \frac{1}{N} \left| \sum_{i=1}^N \xi_i^{M(N)}\right| \le \E^{\mu \otimes \pi} |\xi_1^{M(N)}|\rightarrow 0, \quad N\to \infty,
    \end{eqnarray}
    which implies $N^{-1} \sum_{i=1}^N\xi_i^{M(N)} \rightarrow 0$ in probability, as $N \rightarrow \infty$.

    \textbf{(II)} By the weak law of large numbers it holds that 
    \begin{eqnarray*}
        \left(\frac{1}{N} \sum_{i=1}^N \log Z_f(y_i) - \Lambda(f)\right) \rightarrow 0
    \end{eqnarray*}
    in probability.
\end{proof}

\subsection*{Proof of Theorem \ref{lem: uniqueness cond}}
\begin{proof}
    \textbf{Step 1: Optimality of $L$.}
    By Lemma~\ref{lem: basic lemma} applied at $f = L$, where $\mu_L^y = \mu^y$, we obtain for any $f \in \F_\Phi$
    \begin{eqnarray*}
        \Phi(f) - \Phi(L) = \E^\pi \KL{\mu^y}{\mu^y_f} \ge 0,
    \end{eqnarray*}
    with equality if and only if $\mu^y_f = \mu^y$ for $\pi$-almost every $y$, that is, if and only if $f \sim L$ in $\F_\Phi$ by Lemma~\ref{lem: equivalence class}. Hence $L$ is the strict minimizer of $\Phi$ over $\F_\Phi$.

    \textbf{Step 2: Compactness argument.}
    By Theorem~\ref{thm: bracketing number}, for every $\eta > 0$ the set $\F$ can be covered by finitely many $\eta$-brackets $[l_k, u_k]$, $k = 1, \ldots, K$, with $\norm{u_k - l_k}_{L^{p_0}(\lambda)} \le \eta$. Setting $m_k := \tfrac{1}{2}(l_k + u_k)$, every $f$ in the $k$-th bracket satisfies
    \begin{eqnarray*}
        \norm{f - m_k}_{L^{p_0}(\lambda)} \le \tfrac{1}{2}\norm{u_k - l_k}_{L^{p_0}(\lambda)} \le \tfrac{\eta}{2},
    \end{eqnarray*}
    so $\{m_1, \ldots, m_K\}$ is a finite $\tfrac{\eta}{2}$-net for $\F$. Since $\eta > 0$ was arbitrary, $\F$, and hence its closure $\overline{\F}$, is totally bounded in $L^{p_0}(\lambda)$. As $L^{p_0}(\lambda)$ is complete and $\overline{\F}$ is closed, $\overline{\F}$ is complete; being complete and totally bounded, $\overline{\F}$ is compact in $L^{p_0}(\lambda)$.

    By Lemma~\ref{lem: lambda bounded by Lp}, $\norm{\cdot}_\Phi$ is dominated by $C\norm{\cdot}_{L^{p_0}(\lambda)}$, so the identity map $\big(L^{p_0}(\lambda), \norm{\cdot}_{L^{p_0}(\lambda)}\big) \to \big(\F_\Phi, \norm{\cdot}_\Phi\big)$ is continuous, and $\overline{\F}_\Phi$ is compact with respect to $\norm{\cdot}_\Phi$ as well.
    Fix $\varepsilon > 0$ and define
    \begin{eqnarray*}
        K_\varepsilon := \{ f \in \overline{\F}_\Phi : \norm{f - L}_\Phi \ge \varepsilon \}.
    \end{eqnarray*}
    $K_\varepsilon$ is a closed subset of the compact set $\overline{\F}_\Phi$, hence compact.

    \textbf{Step 3: Attainment.}
    If $K_\varepsilon =\emptyset$, the claim is clear. Hence we assume now $K_\varepsilon \neq \emptyset$.
    Since $\Phi$ is Lipschitz, hence continuous, with respect to $\norm{\cdot}_\Phi$ (by $|\Phi(f) - \Phi(g)| \le \norm{f-g}_\Phi$), it attains its infimum on the compact set $K_\varepsilon$: there exists $f_0 \in K_\varepsilon$ with
    \begin{eqnarray*}
        \Phi(f_0) = \inf_{f \in K_\varepsilon} \Phi(f).
    \end{eqnarray*}
    Since $\norm{f_0 - L}_\Phi \ge \varepsilon > 0$, we have $f_0 \neq L$ in $\F_\Phi$, so by Step 1, $\Phi(f_0) > \Phi(L)$ strictly. Therefore
    \begin{eqnarray*}
        \inf\limits_{\norm{f-L}_\Phi \ge \varepsilon} \Phi(f) = \Phi(f_0) > \Phi(L).
    \end{eqnarray*}
\end{proof}

\subsection*{Proof of Theorem~\ref{thm: M estimation}}
\begin{proof}
    Notice that $N_{[]}(\varepsilon, \F, L_{p_0}(\lambda))$ is finite for every $\varepsilon >0$ by Theorem~\ref{thm: bracketing number} and for $p_0$ given in Lemma~\ref{lem: lambda bounded by Lp}.
    We adapt the proof classical proof (e.g. in \cite[Theorem 3.2]{sen2022lecture}).
Fix $\varepsilon >0$ and choose finitely many $\varepsilon$-brackets $([l_k, u_k])_{k = 1}^K$ that cover $\F$ and with $\lambda|u_k - l_k|^{p_0} < \varepsilon^{p_0}$ for all $k \le K \in \Natural$. Without loss of generality, every $l_k, u_k$ satisfies \eqref{eq: bounds on f}. For any $f \in \F$ there exists a bracket such that $f \in [l_k, u_k]$. We fix $f\in \F$ and the corresponding bracket. 
    We write
    \begin{align*}
        \Psi(f) &:=\E^\lambda f(x, y), &\quad \Psi_N(f) &:= \frac{1}{N} \sum f(x_i, y_i) \\
        \Lambda(f) &:= \E^\pi \left[ \log \E^\mu e^{-f(x, y)} \right], &\quad \Lambda_N(f) &:= \frac{1}{N}\sum_{i = 1}^{N} \log \left[ \frac{1}{M(N)}\sum_{j = i}^{M(N)}  \exp \left( -f(x_{i, j}; y_i) \right) \right]
    \end{align*}
    such that  
    \begin{eqnarray*}
        |\Phi(f) - \Phi_N(f)| &\le& |\Psi(f) - \Psi_N(f)| + |\Lambda(f) - \Lambda_N(f) |.
    \end{eqnarray*}
     Notice since $l_k(x, y)\le u_k(x,y)$ $\lambda$-almost everywhere, and $\Psi, \Lambda$ are increasing or decreasing, it holds for $\Theta \in \{ \Psi, \Lambda\}$
     \begin{eqnarray*}
         |\Theta(f) - \Theta(l_k)| + |\Theta_N(f) - \Theta_N(l_k)| \le |\Theta_N(u_k) - \Theta_N(l_k)| + |\Theta(u_k) - \Theta(l_k)| .
     \end{eqnarray*}
     Thus 
      \begin{align}
      \begin{aligned}\label{eq: unif for f bounds}
        |\Theta_N(f) - \Theta(f)| \le& |\Theta_N(f) -\Theta_N(l_k)| +  |\Theta_N(l_k) - \Theta(l_k)| + |\Theta(l_k) - \Theta(f)| \\
         \le & |\Theta_N(u_k) - \Theta_N(l_k)| +  |\Theta_N(l_k) - \Theta(l_k)| + |\Theta(l_k) - \Theta(u_k)|.
      \end{aligned}
      \end{align}
    The right hand side depends only on the finite indices $k \le K$ but not on $f \in \F$.

    \noindent\textbf{Step 1:}
    Since $\norm{\cdot}_{L^1(\lambda)} \le \norm{\cdot}_{L^{p_0}(\lambda)}$ it holds
    \begin{eqnarray}\label{eq: Psis}
        |\Psi(u_k) - \Psi(l_k)| = \E^\lambda |u_k(x, y) - l_k(x, y)|  \le \norm{u_k - l_k}_{L^{p_0}}  \le \varepsilon.
    \end{eqnarray}
    Using Lemma~\ref{lem: lambda bounded by Lp}, the definition of the brackets, and for $C$ independent of $k$:
    \begin{eqnarray}\label{eq: Lambdas}
        |\Lambda(u_k) - \Lambda(l_k)| \le C \norm{u_k - l_k}_{L^{p_0}}   = C \varepsilon
    \end{eqnarray}
\textbf{Step 2:}
For $\Theta \in \{ \Psi, \Lambda\}$ and each fixed $k = 1, \dots, K$ we have by the law of large numbers and Theorem \ref{lem: conv in p}
\begin{align*}
     |\Theta_N(l_k) - \Theta(l_k)| + |\Theta_N(u_k) - \Theta_N(l_k)|  \to |\Theta(u_k) - \Theta(l_k)| \le (C+1) \varepsilon,
\end{align*}
with convergence in probability. 

\noindent\textbf{Combining the estimates.}
Since the right-hand side of \eqref{eq: unif for f bounds} does not depend on $f\in\F$, we obtain
\begin{equation*}
    \sup_{f\in\F}|\Phi_N(f)-\Phi(f)|
    \;\le\;
    \sup_{k\le K}\,\sum_{\Theta\in\{\Psi,\Lambda\}}
    \Bigl[
        |\Theta_N(u_k)-\Theta_N(l_k)|
        + |\Theta_N(l_k)-\Theta(l_k)|
        + |\Theta(l_k)-\Theta(u_k)|
    \Bigr].
\end{equation*}
By Step~1, the last term satisfies $\sup_{k\le K}\sum_\Theta|\Theta(l_k)-\Theta(u_k)|\le (C+1)\varepsilon$.
By Step~2, the first two terms converge in probability to $\sup_{k\le K}\sum_\Theta|\Theta(u_k)-\Theta(l_k)|\le (C+1)\varepsilon$,
since the supremum is over a finite set.
Hence for any $\delta>0$,
\begin{equation*}
   \mathbb{P}\!\left(\sup_{f\in\F}|\Phi_N(f)-\Phi(f)|>2(C+1)\varepsilon+\delta\right)\;\longrightarrow\;0.
\end{equation*}
Since $\varepsilon>0$ was arbitrary, this gives $\sup_{f\in\F}|\Phi_N(f)-\Phi(f)|\xrightarrow{P}0$.    
\end{proof}


\printbibliography

\end{document}